\documentclass{article}

\usepackage{wrapfig}
\usepackage[main, final, nonatbib]{neurips_2026}

\usepackage[utf8]{inputenc} %
\usepackage[T1]{fontenc}    %
\usepackage{hyperref}       %
\usepackage{url}            %
\usepackage{booktabs}       %
\usepackage{amsfonts}       %
\usepackage{nicefrac}       %
\usepackage{microtype}      %
\usepackage{xcolor}         %

\newif\ifshowrevisions
\showrevisionsfalse
\definecolor{revblue}{RGB}{0,51,160}
\ifshowrevisions
  \newcommand{\revised}[1]{\textcolor{revblue}{#1}}
  \newenvironment{revisedblock}{\color{revblue}}{\ignorespacesafterend}
\else
  \newcommand{\revised}[1]{#1}
  \newenvironment{revisedblock}{}{}
\fi
\usepackage{amsthm}

\newtheorem{theorem}{Theorem}
\newtheorem{lemma}[theorem]{Lemma}
\newtheorem{corollary}[theorem]{Corollary}
\newtheorem{proposition}[theorem]{Proposition}

\theoremstyle{definition}

\theoremstyle{remark}
\newtheorem{remark}[theorem]{Remark}
\usepackage{amsmath}   %
\usepackage{amssymb}   %
\usepackage{amsfonts}  %
\usepackage{amsthm}    %
\usepackage{subcaption} 
\usepackage[acronym]{glossaries}
\makeglossaries
\glsdisablehyper

\newacronym{sysname}{RayNO}{Ray-tracing Neural Operator}
\newacronym{ood}{OOD}{out-of-distribution}
\newacronym{rss}{RSS}{received signal strength}
\newacronym{mc}{MC}{Monte Carlo}
\newacronym{rt}{RT}{ray tracing}
\newacronym{lf}{LF}{low-fidelity}
\newacronym{partialgt}{Partial-GT}{partial ground truth}
\newacronym{hfgt}{HF-GT}{high-fidelity ground truth}
\newacronym{radio}{radio map}{radio map}
\newacronym{sdf}{SDF}{signed distance field}
\newacronym{film}{FiLM}{feature-wise linear modulation}
\newacronym{erm}{ERM}{empirical risk minimization}
\newacronym{fno}{FNO}{Fourier neural operator}
\newacronym{ffno}{F-FNO}{factorized Fourier neural operator}
\newacronym{wno}{WNO}{wavelet neural operator}
\newacronym{gino}{GINO}{geometry-informed neural operator}
\newacronym{pinn}{PINN}{physics-informed neural network}
\newacronym{pino}{PINO}{physics-informed neural operator}
\newacronym{se}{SE}{spectral efficiency}
\newacronym{mcese}{MCESE}{mean cell-edge SE error}
\newacronym{rf}{RF}{radio frequency}
\newacronym{ml}{ML}{machine learning}

\usepackage{mathtools} %

\usepackage{enumitem}  %
\usepackage{bm}        %
\usepackage{mathrsfs}  %

\makeatletter
\@ifundefined{ack}{\newenvironment{ack}{\section*{Acknowledgments and Disclosure of Funding}}{}}{}
\makeatother

\title{Physics-Unrolled Neural Operator for Accelerating \\ High-Fidelity Wireless Ray Tracing}

\author{%
  Rafid Umayer Murshed \\
  University of Illinois Urbana-Champaign \\
  \texttt{rum3@illinois.edu} \\
  \And
  Saif Ur Rahman \\
  University of Illinois Urbana-Champaign \\
  \texttt{saifu2@illinois.edu} \\
  \And
  Mingyue Tang \\
  University of Illinois Urbana-Champaign \\
  \texttt{mt55@illinois.edu} \\
  \And
  Elahe Soltanaghai \\
  University of Illinois Urbana-Champaign \\
  \texttt{elahe@illinois.edu} \\
}

\begin{document}

\maketitle

\begin{abstract}
Site-specific wireless design relies on ray tracing to predict how signal strength varies across a floorplan, but a converged indoor radio map takes tens of seconds to simulate and planning a single building requires thousands of them. A learned surrogate can remove most of that cost, yet training one runs into a second problem: high-fidelity labels are exactly what is too expensive to produce at scale, so the affordable labels come from finite-ray simulations that carry residual Monte Carlo noise. We address this with \gls{sysname}, a physics-unrolled neural operator that reconstructs the converged output of a ray tracer from one cheap low-fidelity run and scene priors. Rather than treating a radio map as a generic image, \gls{sysname} refines the field through three stages matched to reflection, diffraction, and scattering, each supervised by a label of matching physical complexity. We prove that when the label error is conditionally unbiased the learned operator can end up closer to the high-fidelity reference than the labels it was trained on, and we verify that condition directly in the simulator using repeated independent ray seeds. Across diverse floorplans, \gls{sysname} outperforms image-to-image baselines, wireless learning models, monolithic neural operators, and classical multi-fidelity regression on both image-quality and wireless deployment metrics, while reducing the cost of a high-fidelity map from roughly $28$\,s to $0.11$\,s. 
\end{abstract}

\section{Introduction}

Accurate wireless field modeling aims to predict how electromagnetic signals vary across a physical environment. %
The key challenge is that wireless propagation involves coherent multipath interactions with scene geometry and materials, including reflection, diffraction, and scattering, which can produce sharp spatial variations in \gls{rss} \cite{goldsmith2005wireless, tse2005fundamentals}. As a result, practical tasks such as WiFi access-point placement, coverage planning, localization, and network reliability analysis often require high-fidelity radio maps that capture these local variations across space \cite{lu2024newrf,chen2025rfscape,yang2025gsrf}. However, generating such high-fidelity radio maps requires dense ray sampling or detailed propagation modeling, making it computationally expensive\cite{hoydis2023sionnart, modesto2025accelerating}. \revised{This cost adds up quickly in practice, because design studies such as comparing transmitter locations or floorplan layouts require a new simulation for every candidate configuration.}

This motivates the central question of this paper: \textbf{\revised{can a learned model recover the high-fidelity output of a ray tracer from a low-cost run of the same simulator?}} A straightforward direction is to treat this as an image-to-image learning problem~\cite{levie2021radiounet,yapar2022radiomaps}, where a model maps a low-fidelity input radio map to a sharper output. However, such models are often optimized for global visual similarity or average pixel-wise accuracy, which can smooth out sharp local variations caused by multipath. This is especially problematic for radio maps because those local variations are not visual artifacts; they reflect underlying propagation effects such as reflection, diffraction, and scattering, and are critical for coverage and planning \cite{goldsmith2005wireless, tse2005fundamentals}.

\begin{figure}[t]
    \centering
    \includegraphics[width=\linewidth]{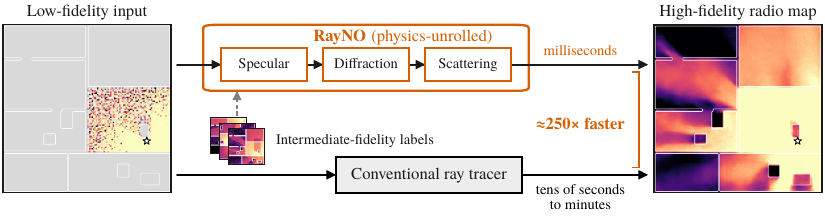}
    \caption{\revised{\gls{sysname} replaces a conventional ray tracer: three physics-unrolled stages (specular, diffraction, and scattering), trained on intermediate-fidelity labels, turn a low-fidelity input into a high-fidelity radio map roughly $250\times$ faster.}}
    \label{fig:concept}
\end{figure}

Another limitation is that learning wireless fields cannot realistically rely on fully high-fidelity labels, since generating them for large training sets is computationally expensive~\cite{modesto2025accelerating}. Instead, the available labels are typically intermediate-fidelity simulations that contain richer spatial structure than the low-fidelity input, but remain noisy and imperfect, as illustrated in Fig.~\ref{fig:concept}. The learning problem is therefore not simply to fit these labels, but to refine coarse fields using the stable structure present in intermediate-fidelity supervision while avoiding simulation noise.

To address these challenges, we propose \gls{sysname}, a physics-unrolled cascade of neural operators for recovering high-fidelity radio maps from low-fidelity ray-tracing outputs. Rather than predicting the full radio map in one step, \gls{sysname} refines the output progressively through multiple operator stages, each focusing on a specific wireless propagation effect: broad reflected coverage, edge-driven diffraction near walls and corners, and fine local scattering fluctuations from objects and multipath. Geometry and scene priors are used to constrain these stages, keeping each refinement grounded in the floorplan, material layout, and transmitter/receiver context. This stage-wise design allows the model to build on the output of the previous operator instead of applying one generic smoothing function to the entire field. This operator formulation further enables field-to-field prediction across different layouts and spatial resolutions, while wireless-physics-aware losses encourage the model to preserve meaningful spatial gradients rather than only matching average pixel values.

We evaluate \gls{sysname} across diverse indoor scenes using simulated radio maps from NVIDIA SionnaRT \cite{sionna}, comparing it against standard vision architectures, general-purpose neural operators, and wireless-specific baselines. \revised{Once trained, \gls{sysname} produces a high-fidelity radio map about $250\times$ faster than running the high-fidelity simulation.} We show that \gls{sysname} can outperform the intermediate-fidelity labels used for training, effectively denoising them without access to high-fidelity reference radio maps during training. Our theoretical analysis identifies a precise condition under which this is possible: the label error must be conditionally unbiased for each scene, meaning that for a fixed floorplan the noise has zero mean and does not systematically bias the field in any direction. The random simulation errors then average out across training samples, while the stable propagation structure shared across scenes remains learnable. This condition holds in our ray-tracing setup, where the intermediate-fidelity labels carry zero-mean Monte Carlo noise.  In summary, the paper makes the following contributions:
\begin{itemize}
    \item[(i)] We introduce \gls{sysname}, a three-stage neural-operator cascade for wireless field refinement that learns from noisy intermediate-fidelity supervision and progressively captures spatial structure associated with reflection, diffraction, and scattering in unseen environments.

    \item[(ii)] We design propagation-aware training objectives that emphasize fine-scale spatial structure in wireless fields, helping preserve details that standard pointwise losses tend to smooth out.

    \item[(iii)] We introduce a deployment-oriented evaluation protocol based on wireless metrics and show that it reveals coverage and planning differences that standard image-quality metrics miss.%

    \item[(iv)] We prove a \revised{denoising theorem that requires no clean reference labels}. It shows that, with enough training scenes, an operator trained on noisy labels is provably closer to the high-fidelity reference than those labels themselves. This extends classical image-denoising results from single-image settings to operator learning across many physical scenes.

\end{itemize}

\section{Background and Related work }
\label{sec:background}

\noindent\textbf{Wireless Fields and Propagation Mechanisms.}
A wireless field describes how electromagnetic signal strength varies across a physical environment. The received signal at each location is the superposition of multiple propagation paths, reflecting off of surfaces and objects. A radio map discretizes this field by assigning a received-signal-strength (RSS) value to each location in the deployment area. These maps support decision-making tasks such as access-point placement, coverage planning, localization, and identifying outage-prone regions.

As illustrated in Fig.~\ref{fig:raytracing-background}, different propagation mechanisms leave distinct spatial signatures: Line-of-sight (LoS) and specular paths form broad coverage patterns, diffraction creates sharp changes near shadow boundaries, and scattering introduces fine local fluctuations. High-fidelity radio map simulation must therefore capture multiple superimposed mechanisms, not merely increase spatial resolution. Full-wave electromagnetic solvers provide detailed approximations of Maxwell-equation behavior but are often too expensive for repeated room- or building-scale planning~\cite{altair_feko,ansys_hfss}, \revised{which has motivated learned surrogates even for single components such as metasurfaces~\cite{murshed2025metafap}}. In practice, scene-scale wireless design often relies on ray tracing, which is more scalable than full-wave simulation but still trades off cost and fidelity through the ray budget, number of bounces, and modeled propagation mechanisms, with runtime increasing by over $25\times$ in some cases~\cite{hoydis2023sionnart,remcom_wireless_insite,zhu2024toward}.

\begin{wrapfigure}{r}{0.55\textwidth}
    \vspace{-11pt}
    \centering
    \includegraphics[width=\linewidth]{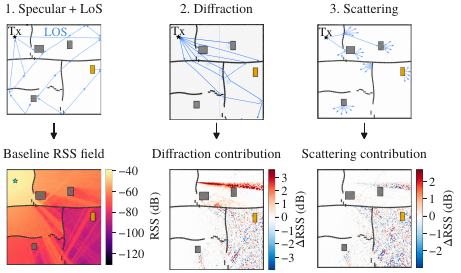}
    \caption{Individual propagation mechanisms jointly shaping the radio map.}
    \label{fig:raytracing-background}
    \vspace{-10pt}
\end{wrapfigure}
Many wireless ray-tracing software packages use \gls{mc} ray tracing, in which propagation is approximated by launching finitely many rays, tracing their interactions with scene geometry, and aggregating their contributions into an \gls{rss} field. With fewer rays, the field is cheaper to compute but only partially converged; with more rays, it becomes richer and more stable at higher cost.

\noindent\textbf{Related Work.}
\label{sec:bg-related}
Our work builds on three strands of scientific \gls{ml}. Neural operators, physics-informed networks, and multi-fidelity methods learn maps between structured fields, often when the supervision is coarse or comes from simulations at different cost~\cite{lu2021deeponet,li2021fno,raissi2019pinns,li2022pino,lu2022multifidelity,howard2023multifidelity}\revised{; in wireless, operator learning has also been combined with explicit propagation-path structure for channel compression~\cite{murshed2026gcno}}. Image denoising work shows that useful predictions can be learned from noisy labels alone, without any clean reference~\cite{lehtinen2018noise2noise,krull2019noise2void,quan2020self2self}. A recent wireless paper uses physics priors to learn the statistical distribution of channel parameters from noisy access-point observations~\cite{bock2025physicsinformed}\revised{, and self-supervised training can keep wireless models robust to imperfect channel estimates~\cite{murshed2024contrastive}}. We draw on each. None of them targets indoor radio maps, and none separates the three propagation effects --- reflection, diffraction, and scattering --- that produce the distinct local patterns wireless coverage planning depends on.

Learning-based wireless surrogates predict signal behavior without exhaustive site surveys. RadioUNet and the Time of Arrival (ToA) dataset family train a U-Net that maps scene geometry directly to a pathloss map~\cite{levie2021radiounet,yapar2022radiomaps}. More recent work fuses sparse on-site \gls{rf} measurements with visual or geometric priors to fit a richer representation of one scene at a time~\cite{chen2024rfcanvas,chen2025rfscape,yang2025gsrf}. Both routes assume something we avoid: the first needs clean training labels and learns one fixed image-to-image map; the second needs \gls{rf} measurements collected on-site at every new scene. We train once across many scenes on noisy ray-traced labels, and predict on unseen floorplans with no on-site \gls{rf} data.

A closely related line uses continuous per-scene representations --- NeRF-style fields, neural ray tracers, and radiance-field methods that fit one model to a single scene from on-site measurements and predict signal at any location in that scene~\cite{orekondy2023winert,zhao2023nerf2,lu2024newrf,bu2025near,rahman2026winerf}. Standard channel models, geometry-based simulators, and differentiable ray tracers provide the physical scaffold these methods sit on~\cite{3gpp38901,jaeckel2014quadriga,hoydis2023sionnart,hoydis2023learningradioenvironments}. In all of them, ray tracing is either the forward simulator that produces the answer or the target a single neural field is fit against. We use ray tracing the other way around: \revised{as the simulator to be accelerated, with its own cheap, noisy outputs serving as supervision across many scenes}. Our model learns a prediction that is closer to a high-fidelity reference than those labels are, and it handles reflection, diffraction, and scattering (Fig.~\ref{fig:raytracing-background}) in three explicit stages rather than smoothing them into one.

\begin{figure}[t]
    \centering
    \includegraphics[width=\textwidth]{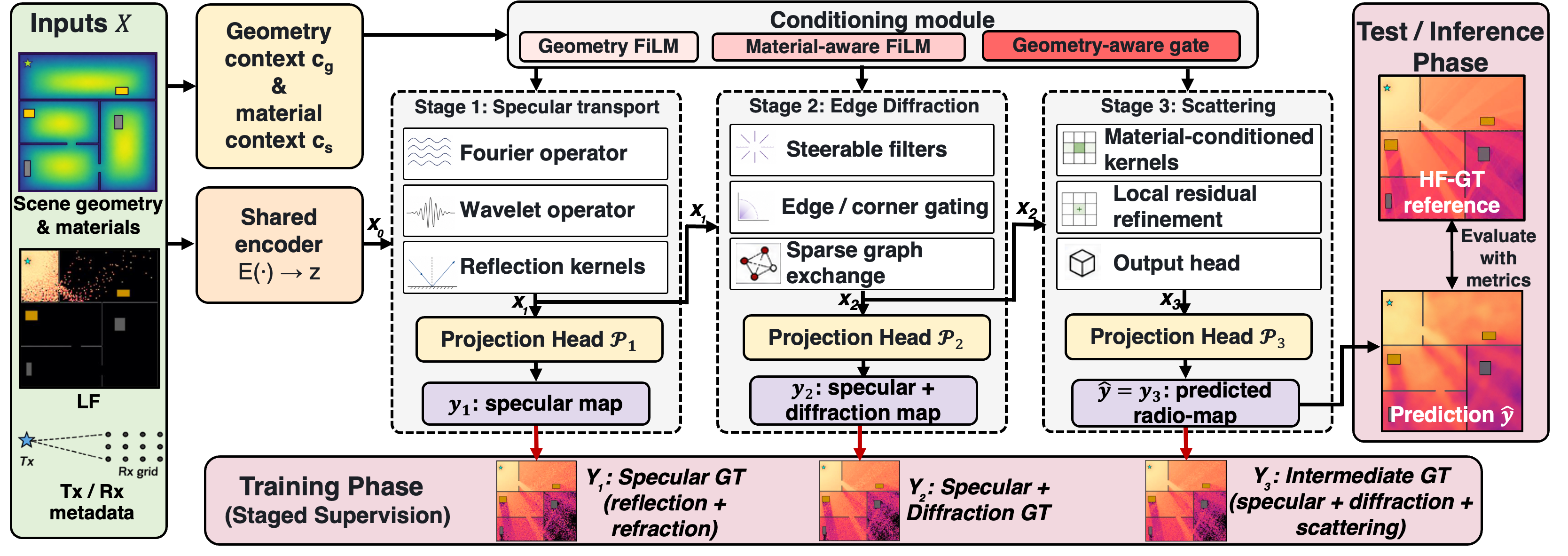}
    \caption{\gls{sysname} estimates high-fidelity radio maps using a three-stage neural-operator architecture that captures different wireless propagation mechanisms.}
    \label{fig:archi}
\end{figure}

\section{Method}
\label{sec:method}

\subsection{Problem Formulation}
\label{sec:problem}
Consider a wireless scene with a fixed transmitter (e.g. a WiFi access point) and receivers distributed over a 2D floorplan $\Omega \subset \mathbb{R}^2$. The corresponding radio map is a 2D matrix of received signal strength values, where each cell $p$ represents a receiver location on the floorplan. We denote by $\mathcal{Y}$ the space of such radio maps. 
For each scene, the model input $X\in\mathcal{X}$ consists of a low-fidelity radio-map (i.e., a coarse simulation-based estimate) and scene priors that describe the floorplan geometry, material information, and coordinate context.
The model is trained using a higher-fidelity but still imperfect radio map as the label, denoted by $Y_{\mathrm{IF\text{-}GT}}\in\mathcal{Y}$, which we call the \emph{Intermediate-Fidelity radio map labels}.
Compared with the low-fidelity input $u$, this label has higher spatial resolution and captures a broader set of wireless interaction mechanisms (e.g. diffraction and scattering). However, it is still an approximate target rather than the final ground truth. For evaluation, we use a separate held-out high-fidelity reference, denoted by $Y_{\mathrm{HF\text{-}GT}}\in\mathcal{Y}$. This \emph{high-fidelity ground-truth radio map reference} represents the desired radio map that the model aims to recover, but it is used only for evaluation and is never seen during training.

The objective is to learn a neural operator $f_\theta:\mathcal{X}\rightarrow\mathcal{Y}$ that maps each input scene representation $X_i$ to a predicted radio map $\hat{Y}_i=f_\theta(X_i)$, where $\theta$ denotes the trainable parameters. For each training scene $i=1,\ldots,n$, the available label is the intermediate-fidelity radio map $Y_{\mathrm{IF\text{-}GT}}^{(i)}$. The model is optimized as
\begin{equation}
\theta^\star
=
\arg\min_{\theta}
\frac{1}{n}
\sum_{i=1}^{n}
\mathcal{L}\big(f_\theta(X_i),Y_{\mathrm{IF\text{-}GT}}^{(i)}\big),
\end{equation}
where $\theta^\star$ denotes the learned parameters and $\mathcal{L}$ is the training loss. Although training uses the intermediate-fidelity ground-truth labels, the goal is for $\hat{Y}_i$ to approach the held-out high-fidelity ground-truth reference.

This is achievable under the assumption that the intermediate-fidelity label is a conditionally unbiased approximation of an ideal radio map for each scene. Let $Y_{\mathrm{oracle}}$ denote this ideal radio map. Then the intermediate-fidelity label can be viewed as
\begin{equation}
\label{eq:unbiasednoise}
Y_{\mathrm{IF\text{-}GT}}^{(i)}
=
Y_{\mathrm{oracle}}^{(i)}
+
\varepsilon_i,
\qquad
\mathbb{E}\!\left[\varepsilon_i \mid X_i\right]=0.
\end{equation}
This means that, for a fixed scene representation $X_i$, the intermediate-fidelity label may contain noise, but the noise does not consistently overestimate or underestimate the ideal radio map. Under this assumption, training on intermediate-fidelity ground-truth labels does not systematically push the model away from the desired mapping. Instead, the zero-mean errors can average out across training examples, allowing the learned operator to approach the ideal mapping. Since $Y_{\mathrm{oracle}}$ is not directly available, we use the held-out high-fidelity reference $Y_{\mathrm{HF\text{-}GT}}$ as its evaluation proxy. We formalize this as a \revised{denoising theorem that requires no clean reference labels} (Appendix~\ref{app:theory}).

\subsection{Input Representation}
\label{sec:features}

For each scene, we separate the model input into three groups as
$X = (u, g, c) \in \mathcal{X}.$
where 

{(i) $u\in\mathcal{Y}$ is the \gls{lf} input radio map}, which provides a coarse estimate of the received signal strength. We also include a binary mask that marks which grid cells contain valid low-fidelity estimates, along with local features that describe whether the nearby signal pattern has a clear dominant direction, as expected from ray-like propagation.

{(ii) $g$ denotes geometry/material prior maps}, which are aligned with the input radio-map grid. These maps encode floorplan and propagation-relevant context, including signed distance fields, obstacle and boundary information, transmitter-distance maps, shadow-region indicators, and material-related maps.

{(iii) $c$ denotes coordinate priors}, which provide spatial and transmitter/receiver context, including receiver-grid coordinates, transmitter-location cues, and transmitter/receiver height information. These coordinate features are Fourier-expanded to provide explicit spatial encoding.

The three input groups are passed through a shared encoder to form the initial latent field. They are also injected directly into the stage-specific operators, where they help form the propagation-aware residual updates at each stage.

\subsection{Three-Stage Physics-Unrolled Operators}
\label{sec:method-stages}

Figure~\ref{fig:archi} illustrates our proposed Physics-Unrolled Neural Operator (\gls{sysname}) that refines a low-fidelity radio map into a high-fidelity radio map prediction through a sequence of wireless-aware stages. The key intuition is that wireless propagation effects appear at different levels of spatial complexity: broad coverage patterns are mainly governed by direct and reflected paths, sharper transitions arise near walls, corners, and shadow boundaries due to refraction and diffraction effects, and fine local variations are caused by scattering and multipath interference. Rather than using a single generic operator to recover all of these structures, \gls{sysname} builds the prediction progressively through three cascaded residual operators:%
\begin{equation}
\label{eq:unroll}
x_k = x_{k-1} + \tilde{\mathcal{H}}_k(x_{k-1};g,c,s_\star),
\qquad y_k=P_k(x_k)\in\mathcal{Y},\qquad k=1,2,3,
\end{equation}
where $x_k$ is the latent field after stage $k$, $\tilde{\mathcal{H}}k$ is the learned residual operator for that stage, and $P_k$ maps the latent field to the stage-wise radio-map prediction $y_k$. The stages build on each other through the residual update: stage $1$ produces an initial reconstruction of the dominant radio-map structure due to specular reflections, stage $2$ refines the remaining edge- and transition-related effects due to diffraction, and stage $3$ adds finer local corrections due to scattering. Each intermediate output $y_k$ is supervised by a stage-specific label $Y_k$ with matching physical complexity. In this fidelity ladder, $Y_1$ and $Y_2$ provide intermediate supervision, while $Y_3 \equiv Y_{\mathrm{IF\text{-}GT}}$ is the final training label used in the problem formulation, and the final prediction is $\hat{Y}=y_3$. The adaptive residual map $s_\star \in [0,1]^\Omega$, predicted from the first-stage latent field $x_1$, indicates where the current prediction still contains unresolved structure. This allows the later stages to focus on residual propagation effects that were not captured by the previous operator, rather than refining the entire field uniformly.
We next define the details of each stage-specific residual operator.

\paragraph{Stage 1: Specular refinement.}
The first stage captures the dominant, large-scale structure of the radio map. This structure is mainly determined by open propagation paths, line-of-sight regions, and strong reflections from major surfaces such as walls. These effects produce smooth, globally coupled coverage patterns. The Fourier basis represents this broad structure compactly, but pure spectral truncation can blur localized features such as wall-cast shadow boundaries. We therefore combine a global spectral operator $\mathcal{G}$ with \revised{a learned multiscale path $\mathcal{W}$, similar in structure to a wavelet transform,} that preserves localized structure, and a reflection-aware local operator $\mathcal{L}_{\mathrm{refl}}$ whose oriented kernels are aligned with the dominant reflection direction inferred from $g$:
\begin{equation}
\tilde{\mathcal{H}}_1(x;g,c) = \mathcal{G}(x;g,c) + \mathcal{W}(x) + \mathcal{L}_{\mathrm{refl}}(x;g,c).
\end{equation}
After this stage, the model estimates $s_\star$ to flag regions needing further refinement. \revised{Appendix~\ref{app:stage_ops} gives the layer-level details of all three stages.}

\paragraph{Stage 2: Diffraction refinement.}
The second stage focuses on regions affected by diffraction, such as areas near wall edges, corners, and doorways, where the floor plan causes abrupt propagation changes. These regions are difficult to recover from Stage 1 because diffraction depends on both the local obstacle geometry and the relationship among nearby diffraction sites. We therefore combine a direction-selective diffraction operator $\mathcal{D}_{\theta}$ with a graph neural operator $\mathcal{N}_{K}$:
\begin{equation}
\tilde{\mathcal{H}}_2(x;g,c,s_\star) = s_\star \odot \left[\mathcal{D}_{\theta}(x;g,c) + \lambda\,\mathcal{N}_{K}(x;g,c)\right].
\end{equation}
Here, $\mathcal{D}_{\theta}$ applies geometry-guided filters using edge, corner, and shadow-boundary cues from $g$, while $\mathcal{N}_{K}$ is a graph neural operator that exchanges information among the top-$K$ candidate diffraction regions.  Multiplication by $s_\star$ restricts the update to regions where Stage 1 leaves unresolved structure.

\paragraph{Stage 3: Scattering refinement.}
The third stage focuses on regions affected by scattering, which occurs when the propagating signal interacts with sharp corners, small geometric irregularities, or materials that disturb the reflected field. Unlike the broad specular structure captured in Stage 1 or the edge-driven diffraction effects refined in Stage 2, scattering produces short-range, fine-scale fluctuations that depend on both local floor-plan geometry and material properties. To capture these effects, we use local kernels $\kappa_s[g]$ generated from the geometry/material input $g$:
\begin{equation}
\tilde{\mathcal{H}}_3(x;g,c,s_\star) = \revised{(\kappa_s[g] * x)} \odot (\alpha_0+\alpha_1 s_\star),
\end{equation}
where \revised{$*$} denotes convolution, $\kappa_s[g]$ is the scattering kernel conditioned on the local geometry/material representation $g$, and $\alpha_0$ and $\alpha_1$ control the baseline and residual-guided strength of the scattering correction.

\subsection{Training Loss Functions}
\label{sec:training-obj}
The model is trained with a multi-term objective that supervises stage-wise predictions, preserves spatial transitions, and aligns residual updates with propagation-related features:
\begin{equation}
\mathcal{L}_{\mathrm{total}}
=
\sum_{k=1}^{3}\lambda_k \mathcal{L}_{\mathrm{Huber}}^{(k)}
+
\lambda_{\mathrm{sob}}\mathcal{L}_{\mathrm{sob}}
+
\sum_{k=1}^{3}\mu_k \mathcal{L}_{\mathrm{trans}}^{(k)}.
\end{equation}

\textit{Stage-matched Huber} $\mathcal{L}_{\mathrm{Huber}}^{(k)}$. Each prediction $y_k$ is compared with its corresponding label $Y_k$ using a Huber loss, which reduces sensitivity to occasional large errors in the intermediate-fidelity labels. A stage-specific confidence map further emphasizes regions where the corresponding propagation effect is expected to be active.

\textit{Distance-modulated Sobolev} $\mathcal{L}_{\mathrm{sob}}$. To preserve sharp spatial transitions near walls, corners, and shadow boundaries, we penalize gradient mismatch in the final prediction:
\begin{equation}
\label{eq:sob-loss}
\mathcal{L}_{\mathrm{sob}}
=
\mathbb{E}
\left[
\left\|
w_{\mathrm{dist}}(X)
\odot
\left(
\nabla y_3 - \nabla Y_3
\right)
\right\|_2^2
\right],
\end{equation}
where $\nabla$ denotes the spatial gradient and $w_{\mathrm{dist}}(X)$ emphasizes geometrically important regions.

\textit{Transport-feature alignment} $\mathcal{L}_{\mathrm{trans}}^{(k)}$. Each residual update is compared with the corresponding label update after both are projected onto directional-energy features:
\begin{equation}
\mathcal{L}_{\mathrm{trans}}^{(k)}
=
\mathbb{E}
\left[
\left\|
\Phi_k(y_k-y_{k-1})
-
\Phi_k(Y_k-Y_{k-1})
\right\|_2^2
\right],
\end{equation}
where $\Phi_k(\cdot)$ extracts local descriptors such as amplitude, coherence, and dominant orientation. This encourages residual corrections to follow propagation structure rather than only pixel-wise error.

The model is trained with a short curriculum that introduces stage-wise supervision first, followed by transport-alignment terms and the weights $\lambda_k$, $\lambda_{\mathrm{sob}}$, and $\mu_k$ control the relative contributions of different loss factors. Implementation details are provided in Appendix~\ref{app:training_protocol}.

\subsection{\revised{Connection to the Theory}}
\label{sec:theory-bridge}
\begin{revisedblock}
The theorem in Appendix~\ref{app:theory} requires the model class to approximate the clean radio map well,
with approximation error $\varepsilon_{\mathrm{approx}}^{2}$ small compared with the label noise. The staged
design targets this term. Let $f^{\star}_{k}$ denote the clean radio map with the propagation effects of
stages $1$ through $k$, and let $f^{\star}_{0}=y_{0}=u$. Each stage then learns one increment
$\Delta_{k}=f^{\star}_{k}-f^{\star}_{k-1}$, and the error of the final prediction is bounded by the errors
on these increments:
\begin{equation}
\label{eq:split-bound}
\big\|y_{3}-f^{\star}_{3}\big\|^{2} \le 3\sum_{k=1}^{3}\big\|\hat{\Delta}_{k}-\Delta_{k}\big\|^{2},
\qquad \hat{\Delta}_{k}:=y_{k}-y_{k-1}.
\end{equation}
Each increment has its own spatial structure, and each stage uses an operator suited to it, whereas a
monolithic operator has to fit all three with one set of kernels. The stage-wise losses keep the denoising
property of a single-target loss (Appendix~\ref{app:stage-loss}).
\end{revisedblock}

\section{Evaluation}
\label{sec:experiments}

\subsection{Setup}
\textbf{Dataset.} We generate the dataset using the NVIDIA Sionna RT ray-tracing simulator~\cite{sionna} over randomly generated floorplans with varying areas, room counts, layouts, and furniture-like objects. For each scene, the low-fidelity input is generated using $10^{4}$ rays with only single-bounce reflections, providing a coarse and incomplete radio map. Models are trained on intermediate-fidelity labels generated with $10^{6}$ rays and full physics, and evaluated against a held-out high-fidelity reference radio maps, generated with the same full-physics setting but $10^{8}$ rays. The \gls{hfgt} reference is never used during training. Test floorplans differ from the training floorplans in room count, layout, and area, allowing us to evaluate generalization to unseen indoor environments.

\textbf{Baseline models.} We compare \gls{sysname} against three groups of baselines: image-to-image regressors, including CNN \cite{lecun2002gradient}, U-Net \cite{ronneberger2015u}, ResNet \cite{he2016resnet}, and ViT \cite{dosovitskiy2021vit}; DL-based wireless channel models, including NERF$^2$\cite{zhao2023nerf2}, RadioUNet~\cite{levie2021radiounet}, GeNeRT~\cite{bian2025genert}, and WiGATr~\cite{orekondy2023winert}; and monolithic neural operators, including FNO~\cite{li2020fourier}, TFNO~\cite{kossaifi2023multi}, UNO~\cite{rahman2022u}, SFNO~\cite{bonev2023spherical}, CodaNO~\cite{rahman2024pretraining}, and WNO~\cite{tripura2023wavelet}. We also include intermediate-fidelity label as a baseline, since it represents the supervisory target used during training and provides a reference for whether a model can improve beyond the noisy labels.

\textbf{Evaluation metrics.} We evaluate predictions using both wireless and image-quality metrics. Wireless metrics focus on whether the predicted map preserves coverage failures, fading behavior, or spectral efficiency. Image-quality metrics, including RMSE, MAE, SSIM, and LPIPS, measure pixel-level accuracy and structural/perceptual similarity. More details on the metrics are provided in Appendix~\ref{app:metrics}.%

\begin{table*}[t]
\centering
\caption{Quantitative radio map prediction results showing that RayNO consistently outperforms all baselines across wireless and computer vision metrics.}
\label{tab:baseline_comparison}
\resizebox{\textwidth}{!}{
\begin{tabular}{ll cccc cccc}
\toprule
 & & \multicolumn{4}{c}{\textbf{Wireless Performance}} & \multicolumn{4}{c}{\textbf{Computer Vision / Generic}} \\
\cmidrule(lr){3-6} \cmidrule(lr){7-10}
\textbf{Model Family} & \textbf{Model} & \textbf{Outage F1} $\uparrow$ & \textbf{Fading Ratio} $\approx 1$ & \textbf{Tail Error (5\%)} $\downarrow$ & \textbf{MCESE} $\downarrow$ & \textbf{RMSE (dB)} $\downarrow$ & \textbf{MAE (dB)} $\downarrow$ & \textbf{SSIM} $\uparrow$ & \textbf{LPIPS} $\downarrow$ \\
\midrule
SIONNA & Training Labels & 0.172 & 1.782 & 1.661 & 1.311 & 14.012 & 5.486 & 0.725 & 0.476 \\
\midrule
Computer Vision & CNN & 0.490 & 0.799 & 1.043 & 1.443 & 4.798 & 3.152 & 0.904 & 0.277 \\
 & U-Net & 0.178 & 0.815 & 0.448 & 1.076 & 3.851 & 2.326 & 0.919 & 0.266 \\
 & ResNet & 0.107 & 0.761 & 0.531 & 1.152 & 4.228 & 2.692 & 0.905 & 0.291 \\
 & ViT & 0.017 & 1.495 & 0.145 & 1.637 & 15.310 & 6.188 & 0.652 & 0.475 \\
\midrule
DL-based Channel Models & NeRF$^2$ & 0.026 & 0.695 & 1.724 & 3.405 & 5.670 & 3.041 & 0.857 & 0.345 \\
 & RadioUNet & 0.000 & 0.314 & 9.356 & 9.348 & 10.327 & 5.960 & 0.840 & 0.296 \\
 & GeNeRT & 0.000 & 0.608 & 1.405 & 2.133 & 5.620 & 3.932 & 0.857 & 0.373 \\
 & WiGATr & 0.002 & 1.517 & 0.559 & 2.078 & 16.305 & 6.807 & 0.643 & 0.487 \\
\midrule
Neural Operators & FNO & 0.025 & 0.750 & 0.690 & 1.605 & 4.045 & 2.362 & 0.903 & 0.313 \\
 & TFNO & 0.018 & 0.715 & 0.805 & 1.848 & 4.125 & 2.280 & 0.912 & 0.280 \\
 & UNO & 0.030 & 0.710 & 0.592 & 1.671 & 4.384 & 2.631 & 0.898 & 0.307 \\
 & SFNO & 0.111 & 0.719 & 0.556 & 1.645 & 4.467 & 2.633 & 0.891 & 0.320 \\
 & CoDA-NO & 0.040 & 0.501 & 2.080 & 4.320 & 7.153 & 4.776 & 0.775 & 0.442 \\
 & WNO & 0.191 & 0.809 & 0.373 & 0.965 & 3.982 & 2.414 & 0.913 & 0.272 \\
\midrule
Proposed Method & \textbf{RayNO (Ours)} & \textbf{0.820} & \textbf{1.013} & \textbf{0.100} & \textbf{0.747} & \textbf{3.471} & \textbf{2.086} & \textbf{0.928} & \textbf{0.229} \\
\bottomrule
\end{tabular}
}
\vspace{-.1in}
\end{table*}
\revised{\textbf{Computational cost.} Including the low-fidelity run it takes as input, one \gls{sysname} prediction takes about $0.11$\,s, against about $28$\,s for a high-fidelity simulation (Appendix~\ref{app:system_cost}).}

\subsection{Results}
\noindent\textbf{Prediction accuracy relative to the high-fidelity reference.} Table~\ref{tab:baseline_comparison} compares all predictions against the held-out \gls{hfgt} reference. Overall, \gls{sysname} gives the most accurate reconstruction across both wireless and image-quality metrics. It reduces tail error to $0.10$, brings the fading ratio close to the ideal value of $1.0$, increases Outage F1 to $0.82$, and reduces RMSE to $3.5$\,dB. These improvements indicate that \gls{sysname} is not only reducing average pixel-wise error, but also more accurately captures coverage-critical regions, such as fading and shadowed areas, where wireless performance can degrade. The first row of the table reports the error of intermediate-fidelity radio maps against high-fidelity references. Since intermediate-fidelity radio maps are the training labels, this row serves as a reference for the noise level in the supervision. The fact that \gls{sysname} outperforms this row shows that the model learns a prediction closer to the high-fidelity reference than the labels it was trained on. \revised{Multi-fidelity co-kriging~\cite{kennedy2000predicting}, the per-scene model GSRF~\cite{yang2025gsrf}, and averaging repeated simulations also fall short of \gls{sysname} (Appendices~\ref{app:baselines} and~\ref{app:if-averaging}).}

\vspace{.3em}
\noindent\textbf{Impact of physics-based model decomposition.}
Several baselines appear competitive under image-quality metrics. For example, U-Net reaches SSIM $0.919$, TFNO reaches $0.912$, and WNO is within $15\%$ of \gls{sysname} on RMSE. However, their wireless metrics are much weaker: U-Net reaches only $0.18$ Outage F1, TFNO drops to $0.018$, and RadioUNet predicts almost no outage regions. This shows that visual similarity or low average error does not necessarily mean the model preserves the propagation structures that matter for wireless deployment. Image-to-image models tend to recover the average field appearance, while monolithic neural operators capture some global structure but smooth over regime-specific effects. By decomposing the prediction into specular, diffraction, and scattering refinements, \gls{sysname} better preserves these wireless-critical structures. As a downstream example, using radio map prediction for access-point placement on a $10{,}000$\,m$^{2}$ floorplan gives $92$\,Mbps throughput for \gls{sysname} versus $16$\,Mbps for CNN, even though the two models differ by only $2.7\%$ in SSIM; details are provided in Appendix~\ref{app:extended-results}.

\begin{figure}[t]
    \centering
    \includegraphics[width=\linewidth]{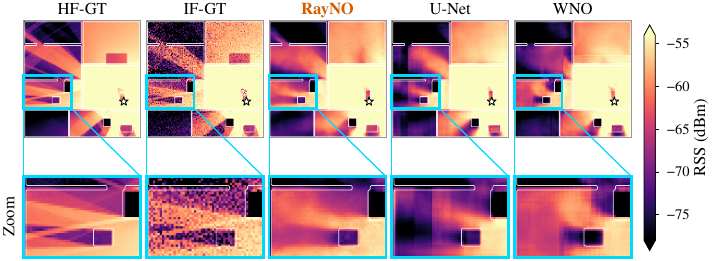}
    \caption{\gls{sysname} recovers fine details from noisy labels and approaches the high-fidelity reference. \revised{Walls are drawn in white and the star marks the transmitter; the bottom row enlarges the boxed region of each map.}}
    \vspace{-.2in}
    \label{fig:qualitative}
\end{figure}

\vspace{.3em}
\noindent\textbf{Ablation study of model components.}
Table~\ref{tab:ablation_study} shows that the main design components contribute to wireless reliability. Removing semantic/material conditioning causes Outage F1 to collapse from $0.82$ to $0.001$, indicating that geometry alone is insufficient for predicting coverage failures; the model also needs to know how different materials and objects affect propagation. \revised{The labels need not be exact: with $10\%$ of them corrupted, \gls{sysname} keeps $91\%$ of its Outage F1 (Appendix~\ref{app:material-robustness}).} The loss design has a similar effect. Removing the distance-aware gradient loss or using only plain Huber supervision reduces Outage F1 to $0.11$ and $0.13$.
The staged architecture is also complementary: specular-only prediction gives Outage F1 $0.21$, adding diffraction raises it to $0.40$, and adding scattering increases it to $0.82$, with consistent improvements in tail error and fading behavior. Other components, including the curriculum, ray descriptors, and module choices, have smaller but consistent effects.

\vspace{.3em}\noindent\textbf{Qualitative analysis.}
Figure~\ref{fig:qualitative} shows an example radio map comparing the predictions from \gls{sysname} and the baselines\revised{ on a test floorplan}. The intermediate-fidelity labels contain some fine propagation signatures, but they are corrupted by Monte Carlo noise and can be sparse in some regions. \revised{U-Net and WNO both keep the broad specular backbone but lose narrow streaks, shadow boundaries, and local fluctuations: U-Net introduces blocky artifacts, and WNO blurs the streaks.} These are the same coverage-critical details captured by the wireless metrics. In contrast, \gls{sysname} recovers these structures with a pattern closer to \gls{hfgt}.%

\begin{table*}[t]
\centering
\caption{Ablation study showing that RayNO benefits from its staged architecture, geometry priors, and propagation-aware losses.}
\label{tab:ablation_study}
\resizebox{\textwidth}{!}{
\begin{tabular}{ll ccccc}
\toprule
\textbf{Category} & \textbf{Variant} & \textbf{Outage F1} $\uparrow$ & \textbf{Fading Ratio} $\approx 1$ & \textbf{Tail Error} $\downarrow$ & \textbf{RMSE (dB)} $\downarrow$ & \textbf{SSIM} $\uparrow$ \\
\midrule
Learning Paradigm & 1A: Single Target (No Curr) & 0.637 & 1.124 & 0.927 & 3.787 & 0.927 \\
 & 1B: Progressive (No Curr) & 0.784 & 0.956 & 0.217 & 3.407 & 0.926 \\
\midrule
Loss Formulation & 2A: w/o Sobolev Gradient & 0.111 & 0.814 & 0.381 & 3.584 & 0.923 \\
 & 2B: w/o Transport (Ray) & 0.740 & 0.910 & 1.273 & 3.512 & 0.916 \\
 & 2C: Unweighted Masked Huber & 0.128 & 0.818 & 0.197 & 3.655 & 0.916 \\
\midrule
Engineered Priors & 3A: w/o Geometric Priors & 0.780 & 0.967 & 0.110 & 3.660 & 0.923 \\
 & 3B: w/o Semantic Priors & 0.001 & 0.731 & 0.959 & 4.107 & 0.919 \\
 & 3C: w/o Ray Descriptors & 0.791 & \textbf{0.997} & 0.179 & 3.656 & 0.927 \\
\midrule
Architecture Scale & 4A: Block 1 Only (Refl) & 0.206 & 0.834 & 0.197 & 3.724 & 0.921 \\
 & 4B: Blocks 1 \& 2 (Refl + Diff) & 0.399 & 0.850 & \textbf{0.096} & 3.603 & 0.927 \\
\midrule
Micro-Architecture & 4C: w/o Graph Corrector & 0.805 & 0.985 & 0.098 & 3.530 & 0.928 \\

 & 4D: w/o Local Wavelet Path & 0.808 & 0.992 & 0.213 & 3.632 & 0.927 \\
\midrule
\textbf{Proposed Full} & \textbf{5: Full RayNO (Ours)} & \textbf{0.820} & 1.013 & 0.101 & \textbf{3.471} & \textbf{0.928} \\
\bottomrule
\end{tabular}
}
\end{table*}

\subsection{Sensitivity Analysis}
\label{sec:sensitivity}

\noindent\textbf{Impact of out-of-distribution (OOD) scene shift.}
We split the test set based on whether the room count and clutter count fall inside or outside the training range, producing four groups: ID, Room-OOD, Clutter-OOD, and Both-OOD. Figure~\ref{fig:sensitive} shows that \gls{sysname} achieves the lowest RMSE across all groups, ranging from $3.31$ to $3.89$\,dB, and maintains Outage F1 between $0.74$ and $0.87$. In contrast, the strongest baseline on each split reaches only $0.34$--$0.53$ Outage F1. This indicates that \gls{sysname} generalizes well to unseen layout and clutter, especially in coverage-critical regions. %

\vspace{.3em}
\noindent\textbf{Impact of intermediate-fidelity label ray budget.}
To evaluate how noisy the training labels can be, we vary the ray budget used to generate intermediate-fidelity radio maps from $50$K to $5$M rays, while keeping the \gls{hfgt} evaluation reference fixed at the $10^{8}$-ray setting. As expected, all models degrade when the training labels become noisier, but Figure~\ref{fig:sensitive} shows that \gls{sysname} degrades more gracefully. Even at $50$K rays, \gls{sysname} maintains Outage F1 around $0.47$, while all baselines fall below $0.05$, and its local-fading ratio remains close to the desired value of $1.0$. This result shows that, as long as the intermediate-fidelity labels retain coherent propagation patterns, the physics-aligned stages can extract those patterns while suppressing much of the noise. \revised{With $10^{6}$ rays per stage label, \gls{sysname} already keeps $95\%$ of the Outage F1 it reaches at $5$M rays, so labels of this size suffice in our setting.}

\vspace{.3em}
\noindent\textbf{Impact of training label noise structure.}
To evaluate the impact of label-noise type, we replace the original intermediate-fidelity labels with synthetic labels created by adding controlled noise to the \gls{hfgt} maps. The models are then evaluated against the uncorrupted \gls{hfgt} reference. 
\gls{sysname} is robust to zero-mean noise types, including IID~\cite{bishop2006pattern}, heteroscedastic~\cite{white1980heteroskedasticity}, and spatially correlated noise~\cite{cressie2015statistics}. Heteroscedastic noise slightly improves RMSE from $3.84$ to $3.62$\,dB and Outage F1 from $0.824$ to $0.834$, suggesting a regularization effect. In contrast, geometry- and distance-dependent biased noises cause a clear failure, increasing RMSE to $12.06$\,dB and reducing Outage F1 to $0.359$. This shows that the model can suppress unbiased noise, but not systematic bias in the training labels.
\revised{Averaging up to $512$ independent simulations of fixed scenes confirms that the label noise of our simulator is zero-mean (Appendix~\ref{app:seed-validation}).}

\vspace{.3em}
\noindent\textbf{Theoretical Justification. } These sensitivity results are supported by our \revised{denoising theorem} in Appendix~\ref{app:theory}. Under the conditionally unbiased noise model in Equation~\ref{eq:unbiasednoise}, the expected loss against the noisy labels has the same minimizer as the loss against the ideal radio map: zero-mean noise adds variance to the supervision but does not shift the optimal mapping. With an expressive enough operator class and enough training scenes, the model therefore learns the propagation function rather than the label noise.

\begin{figure}[t]
    \centering
    \begin{subfigure}{\linewidth}
        \centering
        \includegraphics[width=\linewidth]{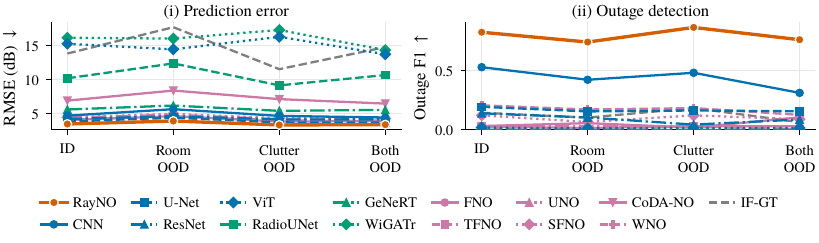}
        \caption{In-distribution (ID) and out-of-distribution (OOD) floorplans.}
        \label{fig:sensitive_ood}
    \end{subfigure}

    \begin{subfigure}{\linewidth}
        \centering
        \includegraphics[width=\linewidth]{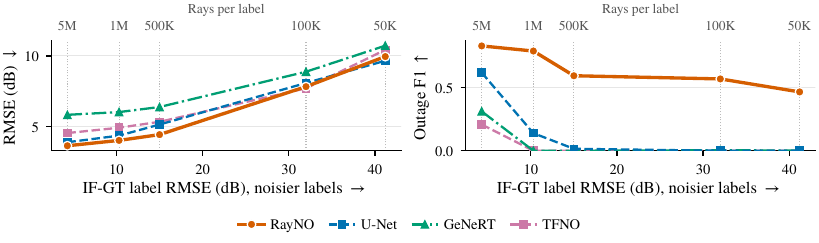}
        \caption{Ray budget of each intermediate-fidelity stage label; the input and the HF-GT reference are fixed.}
        \label{fig:sensitive_ray}
    \end{subfigure}

    \caption{Model sensitivity to structural distribution shifts and training label noise.}
    \label{fig:sensitive}
\end{figure}

\section{Discussion}
\label{sec:limitations}

\vspace{.3em}
\noindent\textbf{Conclusion. }This work presents \gls{sysname}, a physics-unrolled neural operator for high-fidelity radio map prediction from low-fidelity simulation and scene priors. Instead of treating radio maps as generic images, \gls{sysname} decomposes the prediction into propagation-aware stages that progressively refine specular structure, diffraction effects, and scattering-related details. The results show that this design can learn radio map structure that is closer to the held-out high-fidelity reference than the noisy labels used for training, especially on wireless deployment metrics that capture coverage failures, fading behavior, and spectral-efficiency errors. More broadly, the work demonstrates that imperfect simulation labels can still support high-fidelity wireless field prediction when the model architecture and losses are aligned with the underlying physical structures.

\vspace{.3em}
\noindent\textbf{Limitations and Future Work. }
\gls{sysname} depends on accurate geometry and material priors; missing floorplan details or unseen materials can introduce systematic errors. It predicts received signal strength only, not phase, channel impulse responses, or MIMO channel matrices. \revised{Multiple access points can be handled by running the model once per transmitter, but coherent multi-antenna processing requires phase.} Our evaluation covers indoor, single-transmitter scenes with fixed antennas; outdoor scenes, multi-transmitter interference, and antenna patterns remain open. \revised{The model also gives point predictions without uncertainty estimates, which matches deterministic ray tracing but would be useful when material properties are uncertain.} Finally, all supervision and evaluation are based on ray tracing, so accuracy on a real building is limited by how well the simulator models it. \revised{The gap between a simulator and real measurements is mostly systematic, which our noise experiments show the model cannot remove. A practical route is to calibrate the simulator with a few measurements from the target site~\cite{hoydis2023learningradioenvironments} and train \gls{sysname} on the calibrated simulator; validating this on measured data is future work.}

\begin{ack}
We thank the reviewers for their constructive feedback. This work was partially funded by Intel. The authors declare no competing interests.
\end{ack}

\bibliographystyle{IEEEtran}
\bibliography{references}
\appendix
\section{Broader Impact and Limitations}
\label{app:impact}

\paragraph{Intended use.}
RayNO is a learned surrogate for indoor radio-map prediction at sub-6\,GHz
frequencies, intended for tasks where many candidate configurations must be
evaluated against the same scene---access-point placement, coverage planning,
and what-if analysis during 6G indoor deployment. The intended user is a
network designer or a building-scale digital twin, not an end user of a
wireless device. Training and evaluation in this paper use synthetic scenes
and synthetic ray-traced fields, so the work raises no privacy, consent, or
human-subjects concerns at the dataset level.

\paragraph{Foreseeable benefits.}
Replacing repeated high-fidelity ray tracing with a learned operator reduces
the marginal cost of each candidate configuration by several orders of
magnitude; this makes it tractable to consider many more layouts during
planning, which we expect to translate into better coverage and lower power
in deployed networks. The same surrogate can be used inside outer-loop
optimizers for sensor placement or environmental sensing, which would be
prohibitive with full ray tracing.

\paragraph{Foreseeable risks.}
The model is a propagation surrogate; it has no decision-making role and no
direct path to harmful downstream uses that ray tracing itself does not
already enable. The one risk worth naming is over-trust: a planner who
treats the surrogate as ground truth in a regime it was not trained on (e.g.,
outdoor macrocell, dense metallic clutter, or a frequency outside our
training band) may make confident but wrong deployment choices. We address
this in two ways: (i)~the limitations below state the regime explicitly, and
(ii)~Section~\ref{sec:sensitivity} shows where the model breaks under
controlled perturbation, so that users have a concrete picture of the
operating envelope.

\paragraph{Scope and limitations.}
We list these here, framed by us, so that subsequent appendix sections can
be read against a clear scope:

\begin{itemize}
\item \emph{Single frequency.} All training and evaluation data are at
$5.5$\,GHz. Frequency transfer would require either retraining or
explicit frequency conditioning, which we do not study.
\item \emph{Indoor only.} Scenes are bounded $15\,\text{m}\times 15\,\text{m}$
floorplans with concrete outer shell, plasterboard internal walls, and
furniture-scale clutter. We do not evaluate outdoor or macrocell
propagation.
\item \emph{Single transmitter.} Each sample has one active transmitter; we
do not study multi-Tx interference or coordinated transmission.
\item \emph{Scalar power.} The model predicts received signal strength
(RSS), not phase, channel impulse response, angular spectrum, or MIMO
matrices. Phase-coherent extension is left to future work.
\item \emph{Ray-traced supervision.} Both intermediate-fidelity training labels
and HF-GT evaluation references come from the same ray-tracing
engine. Sim-to-real calibration against measured radio maps is out of
scope.
\end{itemize}

These are deliberate scoping choices for a single-paper contribution, not fundamental limits of the architecture. Each can be relaxed by changes to the training distribution and conditioning interface; the operator structure and the noise-averaging mechanism in Section~\ref{app:theory} are independent of them.

\section{Dataset Generation Pipeline}
\label{app:dataset}

This appendix summarizes the dataset construction details needed to reproduce
the experiments. The main text describes the role of the three fidelity levels:
the low-fidelity input, the intermediate-fidelity training labels, and the
held-out high-fidelity reference. Here we specify how the procedural indoor
scenes, ray-tracing fields, and train/test splits are generated.

\subsection{Scene generation}
\label{app:dataset:scenes}

Each scene is a $15\,\mathrm{m}\times 15\,\mathrm{m}$ indoor floorplan with a
$3\,\mathrm{m}$ ceiling. The outer shell, floor, and ceiling are concrete.
Internal room layouts are generated by recursive binary space partitioning
(BSP): rectangular rooms are repeatedly split along their longer side, with a
minimum room dimension of $3\,\mathrm{m}$ and doorway gaps inserted in internal
walls. Figure~\ref{fig:floorplans} shows representative training and test
floorplans from this generator.

Furniture-scale clutter is added by placing axis-aligned cuboid obstacles inside
rooms through rejection sampling. Obstacles have horizontal half-extents sampled
from $\mathrm{Unif}(0.3,0.8)\,\mathrm{m}$ and heights sampled from
$\mathrm{Unif}(0.8,1.5)\,\mathrm{m}$. We enforce a wall margin and a minimum
separation between obstacles to avoid degenerate overlaps. Materials are assigned
from a small indoor alphabet: concrete for the outer shell, plasterboard for
internal walls, and either concrete or metal for furniture. Material parameters
follow the corresponding ITU material profiles and are also stored as
per-pixel material labels for model conditioning.

\subsection{Propagation simulation: LF input, IF-GT labels, and HF-GT reference}
\label{app:dataset:sim}

All radio maps are generated with Sionna RT~\cite{hoydis2023sionnart} at
$5.5\,\mathrm{GHz}$ using isotropic transmit antennas. The receive plane is
rasterized to a $128\times128$ grid, and all received-signal-strength (RSS)
values are stored in dB. For each transmitter and receive-plane configuration,
we generate three fidelity levels:

\paragraph{Low-fidelity input.}
The low-fidelity field $u$ is generated using $10^{4}$ rays per transmitter
with single-bounce specular propagation. This field is used only as an input
scaffold and is never used as a training target.

\paragraph{Intermediate-fidelity labels.}
The training labels are generated using approximately $10^{6}$ rays per
transmitter. We store three staged labels with increasing propagation complexity:
$Y_1$ contains the specular subset, $Y_2$ adds diffraction, and $Y_3$ adds
diffuse scattering. The final intermediate-fidelity label is
$Y_{\mathrm{IF\text{-}GT}}=Y_3$, while $Y_1$ and $Y_2$ provide stage-matched
supervision for the RayNO cascade.

\paragraph{High-fidelity reference.}
The held-out reference is generated with the same propagation mechanisms as
$Y_3$, but with approximately $10^{8}$ rays per transmitter. Thus,
\begin{equation}
\frac{N_{\mathrm{HF}}}{N_{\mathrm{IF}}} \approx 10^{2}.
\label{eq:ray-ratio}
\end{equation}
Under the usual Monte Carlo variance scaling, this makes the HF-GT reference
substantially less noisy than the IF-GT training label. HF-GT is used only for
validation and testing; it is never used to train RayNO or any baseline.

\begin{revisedblock}
\paragraph{\revised{Cost of generating each fidelity level.}}
Table~\ref{tab:gen-cost} reports the runtime, memory, and storage profile of the three fidelity levels. The
asymmetry is the point of the method: the HF-GT reference dominates the one-time validation and testing cost,
while deployment touches only the LF simulation and a $12.1$\,ms forward pass. Complete staged supervision
uses $3\times10^{6}$ rays per sample. Simulator usage is $1.96$\,GB of VRAM per worker; eight parallel workers
produced $2{,}000$ complete samples at $44.4$\,s per sample, and the serialized splits occupy $10.52$\,GB. The
staged labels are generated once and reused, and HF-GT is never a training target. One deployed map therefore
costs about $0.112$\,s against roughly $28$\,s for a direct HF-GT simulation, a reduction near $250\times$.

\begin{table}[h]
\centering
\small
\caption{\revised{Cost of generating each fidelity level, per radio map. Runtimes are measured on the
generation hardware described above.}}
\label{tab:gen-cost}
\begin{tabular}{llll}
\toprule
Artifact & Rays per map & Role & Runtime per map \\
\midrule
LF input $u$                              & $10^{4}$ & Deployment           & $\sim 0.1$\,s \\
Each of $Y_1$, $Y_2$, $Y_3$               & $10^{6}$ & Offline supervision  & $\sim 1.8$\,s \\
HF-GT reference                           & $10^{8}$ & Validation / testing & $\sim 28$\,s \\
\bottomrule
\end{tabular}
\end{table}
\end{revisedblock}

\subsection{Splits and out-of-distribution subgroups}
\label{app:dataset:splits}

The training distribution contains scenes with room counts in
$\{4,5,6,7\}$ and furniture counts in $\{2,3,4,5,6\}$. The test distribution
widens both ranges to room counts in $\{3,4,5,6,7,8\}$ and furniture counts in
$\{0,1,\ldots,8\}$. This lets us evaluate both in-distribution performance and
controlled scene-level distribution shift without changing the underlying
floorplan generator.

For the OOD sensitivity analysis in Section~\ref{sec:sensitivity}, the test set
is partitioned into four disjoint groups:
\begin{itemize}
    \item \emph{ID}: room and furniture counts both lie inside the training range.
    \item \emph{Room-OOD}: room count lies outside the training range, while
    furniture count remains inside.
    \item \emph{Clutter-OOD}: furniture count lies outside the training range,
    while room count remains inside.
    \item \emph{Both-OOD}: both room count and furniture count lie outside the
    training range.
\end{itemize}

The released dataset contains the low-fidelity input, staged IF-GT labels,
HF-GT references, geometry maps, material labels, transmitter/receiver context,
and scene identifiers for each sample.

\paragraph{Dataset Link:}
\href{https://drive.google.com/drive/folders/1pC0gqyNScvyUifFPTnnNCmLaEeznYzvh?usp=sharing}{RayNO Procedural Indoor RSS Dataset}

\begin{figure}[t]
    \centering
    \includegraphics[width=\linewidth,trim={0.1cm 0.2cm 0.2cm 0.2cm}, clip]{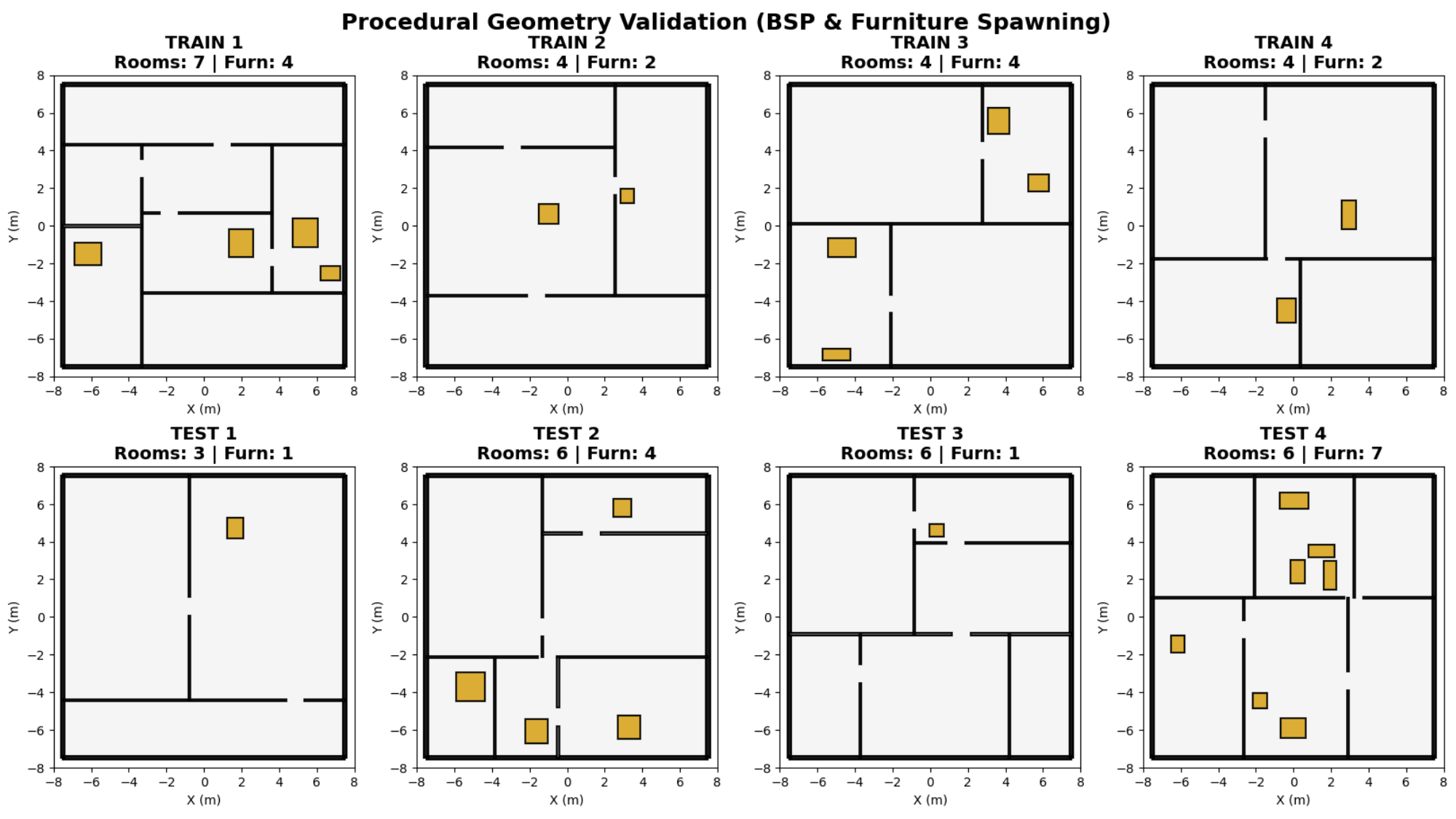}
    \caption{\textbf{Representative training and test floorplans.} Environments are procedurally generated via recursive binary space partitioning (BSP) and randomized furniture spawning. The test set introduces structural distribution shifts by expanding the permitted range of room and obstacle counts beyond the training distribution.}
    \label{fig:floorplans}
\end{figure}

\begin{figure}[t]
    \centering
    \begin{subfigure}{\linewidth}
        \centering
        \includegraphics[width=\linewidth, trim={0.1cm 0.2cm 0.2cm 0.2cm}, clip]{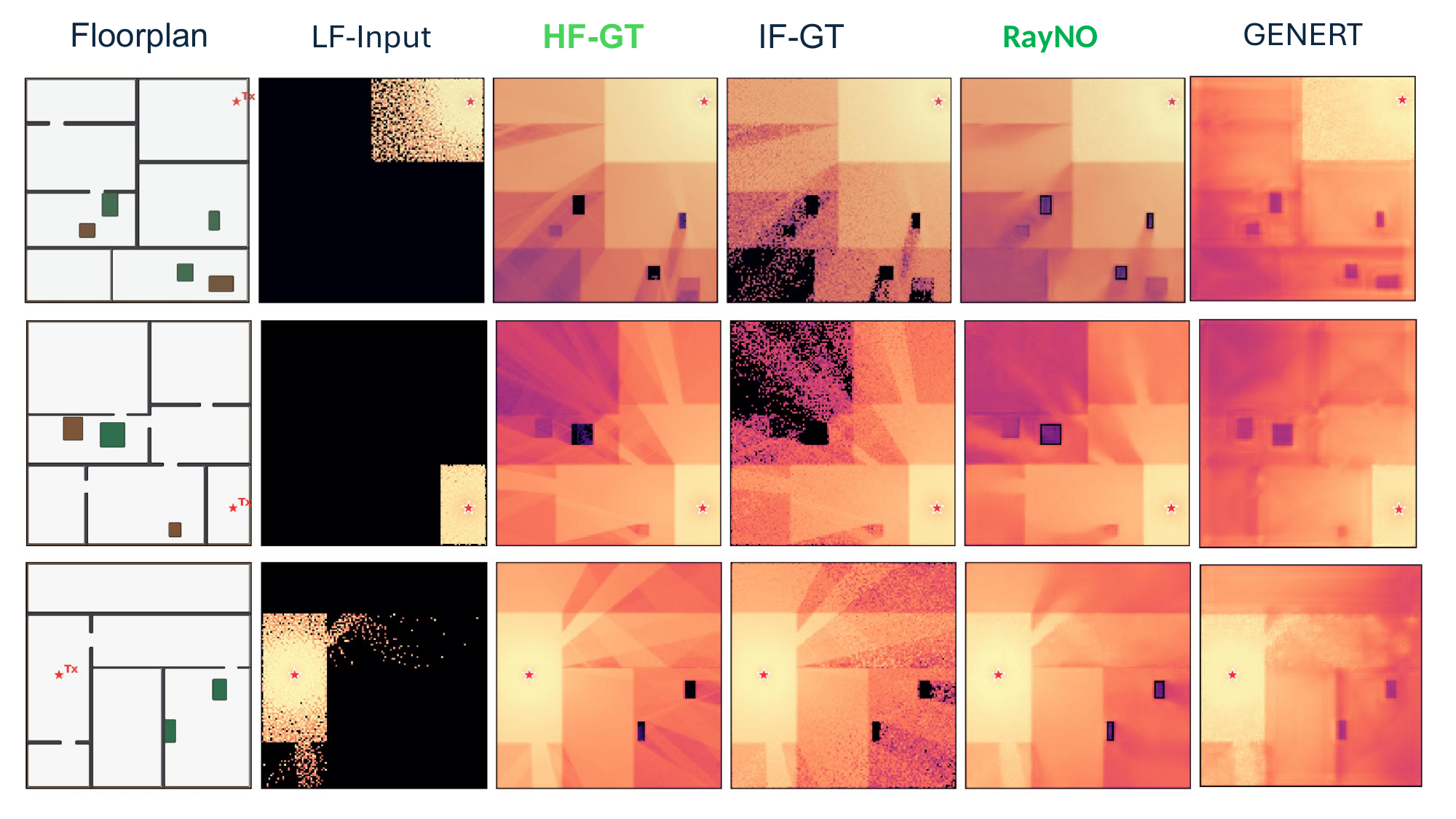}
        \caption{}   %
    \end{subfigure}
    \vspace{0.2cm}   %
    \begin{subfigure}{\linewidth}
        \centering
        \includegraphics[width=\linewidth, trim={0.1cm 0.2cm 0.2cm 0.2cm}, clip]{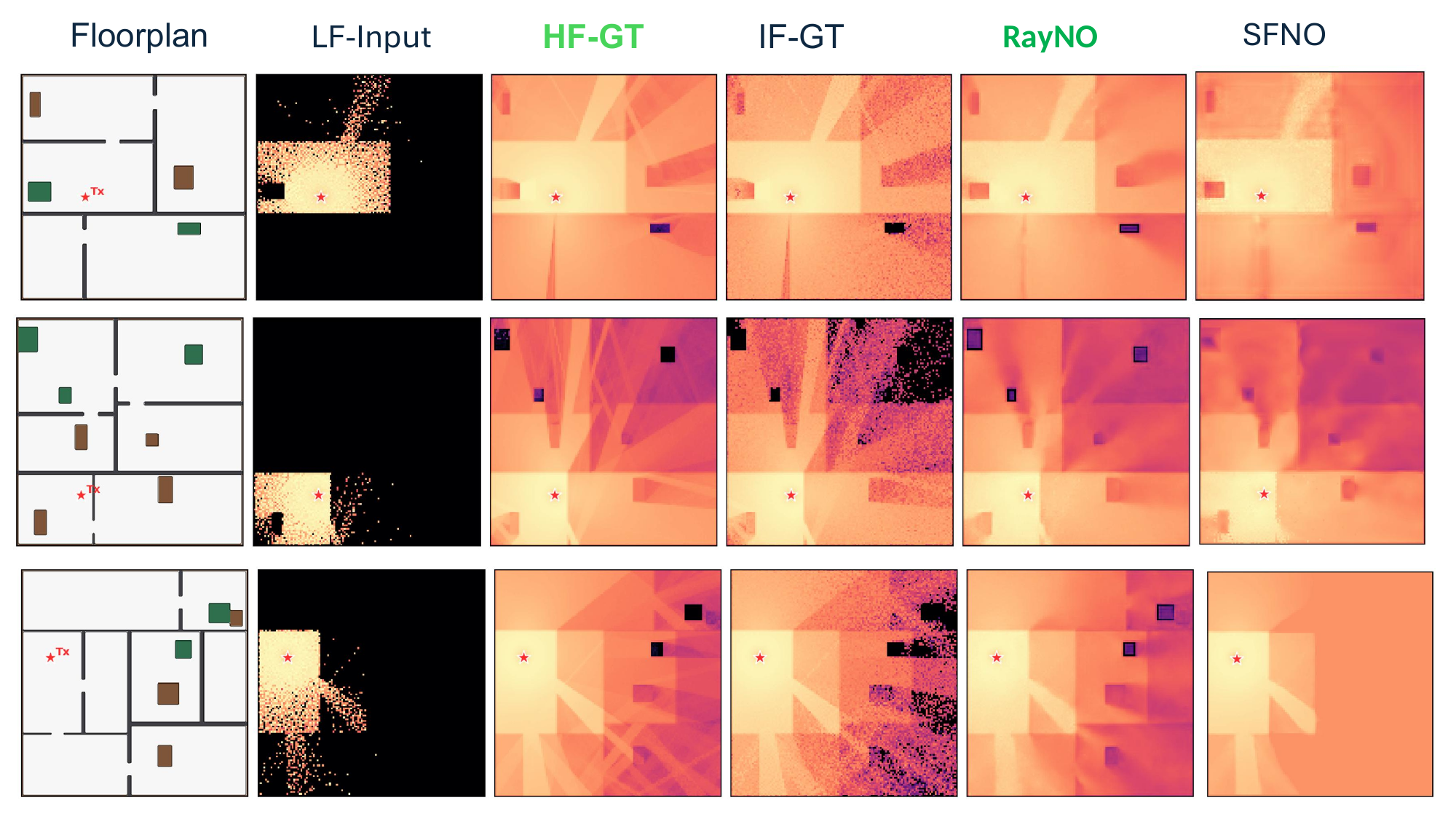}
        \caption{}
    \end{subfigure}
    \caption{\textbf{Qualitative evaluation of radio map reconstruction.} RayNO accurately recovers the High-Fidelity Ground Truth (HF-GT) field from sparse LF-Inputs. Baselines such as GeNeRT (a) and SFNO (b) struggle to preserve these sharp, deployment-critical details and exhibit significant over-smoothing.}
    \label{fig:pred_comparison}
\end{figure}

\section{Model Architecture}
\label{app:architecture}

The main text describes RayNO as a three-stage residual operator for
specular transport, diffraction, and scattering. This appendix records the
architectural details needed to reproduce the reported model, without
repeating the full motivation from Section~\ref{sec:method}.

\subsection{Overall structure}
\label{app:arch_overview}

RayNO applies three residual operators in sequence. Let $x_0$ denote the
encoded input field and $x_k$ the latent state after stage $k$. The model
uses
\begin{equation}
x_k \;=\; x_{k-1} \;+\; \widetilde{H}_k\!\big(x_{k-1};\, g, c, s^\star\big), \qquad y_k \;=\; \mathcal{P}_k(x_k), \qquad k = 1,2,3,
\label{eq:residual_unroll}
\end{equation}
where $\widetilde{H}_k$ is the stage-$k$ residual operator,
$\mathcal{P}_k$ is the stage readout head, and
$s^\star \in [0,1]^\Omega$ is the adaptive sizing field used to restrict
later corrections to regions that still need refinement. The final
prediction is $\widehat{Y}=y_3$.

The residual form keeps each stage close to an identity correction at
initialization and makes the cascade stable: later stages can add localized
updates without rewriting the entire field. The model has roughly $2.5$M
trainable parameters, with most capacity allocated to the first stage because
it must recover the global field structure.

\begin{table}[h]
\centering
\small
\caption{Parameter budget by component. Counts are obtained by instantiating the model used for the reported results and summing over the corresponding submodules. Totals are rounded to the nearest thousand.}
\label{tab:param_summary}
\begin{tabular}{lr}
\toprule
Component & Trainable parameters \\
\midrule
Shared encoder, conditioning, and projection heads & $\sim 187$K \\
Stage 1 (specular)   & $\sim 1.56$M \\
Stage 2 (diffraction)& $\sim 614$K \\
Stage 3 (scattering) & $\sim 147$K \\
\midrule
Total & $\sim 2.50$M \\
\bottomrule
\end{tabular}
\end{table}

\subsection{Encoder and conditioning}
\label{app:encoder}

The shared encoder maps all input channels to a full-resolution latent field
$x_0$. Its inputs are the low-fidelity RSS field, valid-pixel mask,
Fourier-expanded coordinates, Fourier-expanded transmitter-distance features,
and a compact ray-evidence descriptor extracted from the low-fidelity input.
The encoder is a single $3\times 3$ convolution with normalization and a
smooth nonlinearity, producing a latent width of $128$.

Two conditioning streams are reused by all stages. The geometry stream
combines SDF, transmitter distance, line-of-sight cues, wedge/corner cues, and
occupancy. The semantic stream combines free-space/occupied indicators with
per-material labels over \{free space, plasterboard, concrete, metal\}. These
streams are applied through lightweight modulation blocks rather than by
repeated concatenation.

We use three recurring conditioning modules:
\begin{itemize}[leftmargin=1.5em]
    \item \emph{Geometry FiLM}: a bounded Feature-wise Linear Modulation block
    conditioned on geometry context~\cite{perez2018film}.
    \item \emph{Material-aware FiLM}: the same mechanism conditioned on
    semantic/material context, used where the operator should behave
    differently across materials.
    \item \emph{Geometry-aware gate}: a pointwise gate that scales residual
    updates before they are added back to the latent field.
\end{itemize}
The modulation magnitudes and gate biases are initialized conservatively so
that the network begins close to the low-fidelity input and learns corrections
gradually.

\subsection{Stage operators}
\label{app:stage_ops}

\paragraph{Stage 1: specular refinement.}
Stage 1 builds the first usable field estimate from the low-fidelity input. It
contains three paths: a factorized Fourier path for long-range coupling in the
spirit of FNO~\cite{li2021fno}, a local wavelet-style path for localized
detail inspired by WNO~\cite{tripura2023wavelet}, and a reflection-aware local
branch that uses geometry-derived reflection directions to select among
oriented filters. The three outputs are fused into a residual update, and the
stage readout is added to the low-fidelity field so that $y_1$ remains anchored
to the simulator scaffold.

\begin{revisedblock}
\paragraph{\revised{Stage 1 in detail: what is learned and what is engineered.}}
Each Stage-1 global block applies three paths in parallel. (i) Spectral mixing takes separate
one-dimensional FFTs along the horizontal and vertical axes, keeps $20$ modes per axis, and applies learned
complex weights. (ii) The multiscale path follows an analysis--filter--synthesis chain,
\begin{equation}
\label{eq:wavelet-path}
\mathrm{PixelUnshuffle}_{2}
\;\rightarrow\; 1\times1\,(512\!\rightarrow\!192)
\;\rightarrow\; 3\times3\ \text{depthwise}
\;\rightarrow\; 1\times1\,(192\!\rightarrow\!512)
\;\rightarrow\; \mathrm{PixelShuffle}_{2},
\end{equation}
which forms a coarser-scale representation, learns localized coarse-to-fine filters there, and reconstructs
full-resolution features. We describe this as wavelet-style rather than as a wavelet transform: the analysis
and synthesis steps are exact rearrangements and the filters are learned, so the multiscale basis is
data-adaptive and no wavelet family is selected. (iii) A pointwise $1\times1$ branch carries the remainder.
Stage 1 additionally uses eight learned $9\times9$ reflection-aligned filter banks and a learned orientation
selector.

Its $1{,}556{,}889$ trainable parameters comprise the spectral weights, the multiscale path, the directional
filters, the conditioning and fusion layers, the residual and transport heads, and the gating weights. The
quantities that steer these operators are engineered rather than learned: signed-distance normals, reflection
directions, boundary cues, directional masks, and ray-evidence kernels are all computed from $g$ before the
forward pass.

\paragraph{\revised{Stage 2 in detail: the graph operator.}}
The graph corrector $\mathcal{N}_K$ acts on $K=192$ grid locations selected by the combined adaptive-sizing
and geometry scores. Their $128$-dimensional grid features are projected to $32$-dimensional node states
$h_i$. Edges join spatial neighbours among the selected nodes and carry
\begin{equation}
\label{eq:graph-edge}
e_{ij} = \big[\Delta x_{ij},\; \Delta y_{ij},\; d_{ij},\; \Delta w_{ij},\; \Delta c_{ij},\;
\Delta b^{\mathrm{ISB}}_{ij},\; \Delta b^{\mathrm{RSB}}_{ij}\big],
\end{equation}
that is, the relative position and distance between the two nodes together with differences in wall, corner,
and incident- and reflection-shadow-boundary cues. Two rounds of message passing then apply
\begin{equation}
\label{eq:graph-mp}
m^{(t)}_{ij} = M_{\theta}\big([h^{(t)}_{i}, h^{(t)}_{j}, e_{ij}]\big),
\qquad
h^{(t+1)}_{j} = \mathrm{LN}\Big(h^{(t)}_{j} + U_{\theta}\big([\,h^{(t)}_{j},\; \textstyle\sum_{i:(i,j)\in E} m^{(t)}_{ij}\,]\big)\Big).
\end{equation}
The updated states are projected back to $128$ channels, scattered to their grid positions, gated by
$s^\star$, and fused with the directional diffraction correction. The branch is zero-initialized, so it
begins as the identity and the model recovers the purely directional Stage 2 if training finds no use for it.
It holds $21{,}472$ parameters, under $0.9\%$ of the model, which is why removing it costs so little
(Table~\ref{tab:ablation_study}, row 4C).
\end{revisedblock}

\paragraph{Adaptive sizing field.}
After Stage 1, a small convolutional head predicts
$s^\star\in[0,1]^\Omega$. This field gates the Stage 2 and Stage 3 residual
updates. It is not a hand-coded mask; it is learned from the Stage-1 latent,
geometry features, and ray-evidence descriptor. Its purpose is to let later
stages focus on unresolved regions instead of refining the whole map uniformly.

\paragraph{Stage 2: diffraction refinement.}
Stage 2 targets edge-driven structure near corners, doorways, and shadow
boundaries. It combines a wedge-conditioned directional filtering operator with
a sparse graph corrector over geometrically important pixels. The directional
operator is aligned with geometric diffraction intuition~\cite{keller2016geometrical},
while the graph corrector allows nearby hard regions to exchange information. \revised{Removing the graph corrector lowers Outage F1 only from $0.820$ to $0.805$ (Table~\ref{tab:ablation_study}), so a lighter model without it remains a practical option.}
Both components are gated by $s^\star$ before being added to the latent state.

\paragraph{Stage 3: scattering refinement.}
Stage 3 models the remaining fine, material-dependent local corrections. It
uses depthwise-separable convolutions whose channel-wise scale and shift are
conditioned on the per-pixel material embedding:
\begin{equation}
h \;\leftarrow\; \mathrm{GN}\!\left(\mathrm{DWConv}(h)\right) \cdot \big(1 + \gamma(m)\big) \;+\; \beta(m),
\label{eq:scatter_layer}
\end{equation}
where $m$ is the material embedding map, and $\gamma(m),\beta(m)$ are
pointwise functions of that map. This is a local FiLM-conditioned convolution
\cite{perez2018film}. The update is again gated by $s^\star$ and by the
geometry-aware gate.

\subsection{Prediction heads and curriculum}
\label{app:heads_curriculum}

The stage readout heads are intentionally small. $\mathcal{P}_1$ and
$\mathcal{P}_2$ are $1\times1$ projections, while $\mathcal{P}_3$ adds a small
geometry-conditioned refinement before the final output. Final layers are
initialized near zero so that the initial prediction is close to the
low-fidelity input.

Training activates the stages progressively. Stage 1 is trained first, Stage 2
is then enabled with its intermediate label $Y_2$, and Stage 3 is finally
enabled with $Y_3=Y_{\mathrm{IF\text{-}GT}}$. The transport-feature auxiliary
losses are introduced with a linear ramp only after the main reconstruction
objective is established. This curriculum follows the physical ordering of the
cascade and is used consistently for all reported RayNO runs.
\section{Loss Function and Training Protocol}
\label{app:training_protocol}

The main text gives the high-level training objective. This appendix specifies
the exact loss terms, weights, optimizer, and compute setting used for the
reported RayNO results.

\subsection{Confidence-modulated reconstruction loss}
\label{app:conf_weight}

The intermediate-fidelity labels are finite-budget Monte Carlo ray-tracing
outputs, so their reliability varies across space. We therefore use an
input-dependent weight that gives extra emphasis to geometrically important
regions while masking invalid pixels:
\begin{equation}
w_{\text{conf}}(p) \;=\; \big(1 + \mathrm{conf}(p)\cdot \mathrm{geom}(p)\big)\cdot \mathbb{1}_{\text{valid}}(p),
\label{eq:conf_weight}
\end{equation}
with
\begin{align}
\mathrm{conf}(p) &= \frac{1}{1 + \big(d(p)/r_{\text{scale}}\big)^2}, \label{eq:conf_factor}\\
\mathrm{geom}(p) &= \alpha\, \exp\!\big(-\kappa\, |\,\mathrm{SDF}(p)\,|\big) + \beta\, \mathrm{boundary}(p).\label{eq:geom_factor}
\end{align}
Here $d(p)$ is the transmitter distance, $\mathrm{SDF}(p)$ measures distance
to geometry, $\mathrm{boundary}(p)$ marks edge/corner/shadow-boundary regions,
and $\mathbb{1}_{\text{valid}}$ selects valid IF-GT pixels. The weight depends
only on the input and is detached from the optimization graph.

Each stage prediction $y_k$ is supervised by its matching label $Y_k$ using a
weighted Huber loss:
\begin{equation}
\mathcal{L}^{(k)}_{\text{rec}} \;=\; \frac{\sum_{p} w_{\text{conf}}(p)\, \rho_\delta\!\big(y_k(p) - Y_k(p)\big)}{\sum_{p} w_{\text{conf}}(p) + \varepsilon},
\label{eq:huber_loss}
\end{equation}
where $\rho_\delta$ is the Huber penalty. The normalization by total active
weight keeps the loss scale comparable across scenes.

\subsection{Sobolev gradient loss}
\label{app:sobolev_loss}

To discourage over-smoothed predictions, the final prediction is also
supervised in the gradient domain. Let $D_x$ and $D_y$ be fixed Sobel
operators. We use
\begin{equation}
\mathcal{L}_{\text{sob}} \;=\; \frac{\sum_{p} w_{\text{sob}}(p)\, \big[\rho_\delta(D_x y_3 - D_x Y_3) + \rho_\delta(D_y y_3 - D_y Y_3)\big]}{\sum_p w_{\text{sob}}(p) + \varepsilon},
\label{eq:sobolev_loss}
\end{equation}
with
\begin{equation}
w_{\text{sob}}(p) \;=\; \mathrm{conf}(p) \cdot \mathbb{1}_{\text{valid}}(p).
\label{eq:sob_weight}
\end{equation}
We apply this term only at Stage 3, because the earlier stage targets are
intentionally smoother than the final IF-GT field. Appendix~\ref{app:theory}
shows that input-measurable weighted Sobolev losses preserve the same
population target under conditionally unbiased label noise.

\subsection{Transport-feature auxiliary loss}
\label{app:transport_loss}

We additionally supervise the residual updates $y_k-y_{k-1}$ using fixed
directional descriptors. Let $\Phi_k(\cdot)$ summarize local amplitude,
coherence, and dominant orientation from a bank of oriented filters. The
transport-feature loss is
\begin{equation}
\mathcal{L}^{(k)}_{\text{trans}} \;=\; \frac{\sum_p w^{(k)}_{\text{ray}}(p)\, \mathrm{conf}(p)\, \rho_{\delta'}\!\big(\Phi_k(y_k - y_{k-1}) - \Phi_k(Y_k - Y_{k-1})\big)}{\sum_p w^{(k)}_{\text{ray}}(p)\, \mathrm{conf}(p) + \varepsilon}.
\label{eq:transport_loss}
\end{equation}
The descriptor $\Phi_k$ is fixed, not learned. Thus this term does not define
a second output target; it only encourages each residual update to follow the
directional structure of the corresponding label update.

\subsection{Composite objective}
\label{app:total_loss}

The full objective is
\begin{equation}
\mathcal{L} \;=\; \sum_{k=1}^{3} \lambda_k\, \mathcal{L}^{(k)}_{\text{rec}} \;+\; \lambda_{\text{sob}}\, \mathcal{L}_{\text{sob}} \;+\; \tau(t)\sum_{k=1}^{3} \mu_k\, \mathcal{L}^{(k)}_{\text{trans}},
\label{eq:total_loss}
\end{equation}
where $\tau(t)\in[0,1]$ is the transport-loss ramp used during curriculum
training. The same coefficients are used for every reported RayNO result.

\begin{table}[h]
\centering
\small
\caption{Loss hyperparameters used for all reported runs. ``Distance scale'' refers to $r_{\text{scale}}$ in Eq.~\eqref{eq:conf_factor}, in physical units relative to the $15\,\text{m}\times 15\,\text{m}$ floorplan.}
\label{tab:loss_hparams}
\begin{tabular}{llr}
\toprule
Term & Symbol & Value \\
\midrule
Stage-1 reconstruction         & $\lambda_1$    & $0.55$ \\
Stage-2 reconstruction         & $\lambda_2$    & $1.00$ \\
Stage-3 reconstruction         & $\lambda_3$    & $1.60$ \\
Sobolev gradient               & $\lambda_{\text{sob}}$ & $0.15$ \\
Stage-1 transport feature      & $\mu_1$        & $0.20$ \\
Stage-2 transport feature      & $\mu_2$        & $0.28$ \\
Stage-3 transport feature      & $\mu_3$        & $0.40$ \\
Reconstruction Huber transition & $\delta$      & $1.0$ \\
Transport Huber transition     & $\delta'$      & $0.05$ \\
Geometry weighting (wall-proximity coefficient) & $\alpha$ & $0.6$ \\
Geometry weighting (boundary coefficient) & $\beta$ & $0.4$ \\
SDF decay rate                 & $\kappa$       & $4.0$ \\
Distance scale                 & $r_{\text{scale}}$ & $\sim 5\,\text{m}$ \\
\bottomrule
\end{tabular}
\end{table}

\subsection{Optimizer, precision, and compute}
\label{app:optimizer_compute}

RayNO is trained with AdamW~\cite{loshchilov2017decoupled} using
$\beta_1=0.9$, $\beta_2=0.95$, weight decay $10^{-4}$, and base learning rate
$8\times10^{-4}$. The learning rate uses a three-epoch linear warm-up followed
by cosine decay. Biases, normalization parameters, and one-dimensional
parameters are excluded from weight decay.

Training uses an effective batch size of $64$ scenes, mixed precision with
\texttt{bfloat16}, and gradient clipping at global norm $1.0$. A fixed random
seed is used for Python, NumPy, and PyTorch. The main results use one seed;
robustness to label noise is evaluated separately in Section~\ref{sec:sensitivity}.

All reported RayNO runs use a single H200-class NVIDIA GPU. A full run uses
35 epochs over the indoor dataset in Appendix~\ref{app:dataset} and takes
approximately 1.5 GPU-hours. No multi-GPU or multi-node training is required.
\begin{revisedblock}
\subsection{\revised{Inference cost and system-level measurements}}
\label{app:system_cost}

Table~\ref{tab:system-metrics} reports the deployment-side profile. Measurements use \texttt{bfloat16} on an
H200 NVL with ten warm-up passes followed by fifty synchronized timed runs. The curriculum, the staged labels,
and the multi-term objective are training-time constructs only and add no modules at inference; neither the
IF-GT labels nor the HF-GT reference is required to predict a map.

\begin{table}[h]
\centering
\small
\caption{\revised{System-level measurements for \gls{sysname}. Compute and latency are per $128\times128$
radio map.}}
\label{tab:system-metrics}
\begin{tabular}{ll}
\toprule
Quantity & Value \\
\midrule
Training                        & $1.5$ GPU-hours; $35$ epochs; one H200-class GPU \\
Parameters / model size         & $2.50$M / $12.0$\,MB \\
Compute per map                 & $119.24$ GFLOPs ($59.62$ GMACs) \\
Latency / throughput, batch $1$ & $12.114$\,ms per map / $82.55$ maps\,s$^{-1}$ \\
Throughput, batch $32$          & $157.56$ maps\,s$^{-1}$ \\
Peak inference VRAM, batch $1$  & $294$\,MiB allocated; $362$\,MiB reserved \\
Peak inference VRAM, batch $32$ & $7{,}925$\,MiB allocated; $10{,}836$\,MiB reserved \\
\midrule
End-to-end per map              & $\sim 0.1$\,s LF simulation $+$ $12.1$\,ms inference $\approx 0.112$\,s \\
Direct HF-GT simulation         & $\sim 28$\,s \\
\bottomrule
\end{tabular}
\end{table}
\end{revisedblock}

\section{Baselines and Fair-Comparison Protocol}
\label{app:baselines}

Table~\ref{tab:baseline_comparison} compares RayNO against three baseline
families: image-to-image regressors, wireless-specific learning models, and
monolithic neural operators. The purpose of this appendix is not to reintroduce
these well-known architectures, but to document the comparison protocol and the
minimal implementation choices needed for reproducibility.

\subsection{Shared protocol}
\label{app:baseline_protocol}

All baselines use the same train/validation/test split, the same IF-GT
supervision, and the same HF-GT evaluation reference as RayNO. Each model is
sized to the same capacity band as RayNO: the target is $2.5$M trainable
parameters, and all realized counts lie within $\pm 10\%$ of that target. When a
published architecture has a default size outside this range, we adjust only
width, depth, mode count, rank, or patch size as appropriate, while preserving
the architecture's standard design.

All baselines receive the same information content as RayNO: the low-fidelity
field $u$, geometry/material features $g$ including SDF, transmitter distance,
occupancy, and material identity, and coordinate features $c$ including receiver
coordinates and transmitter/receiver elevations. These inputs are formatted
according to each model family's native interface.

Hyperparameters are selected using the same search protocol for every baseline.
We tune learning rate in
$\{2\times10^{-4},6\times10^{-4},1.2\times10^{-3}\}$ and weight decay in
$\{10^{-5},10^{-4}\}$, giving six candidates per model. We use successive
halving~\cite{li2018hyperband}: six candidates are trained briefly, the best
three are promoted, and then the best two are trained longer. The selected
configuration is the one with lowest validation RMSE.

The final training recipe is also shared across baselines: AdamW, batch size
$64$, mixed precision with \texttt{bfloat16}, gradient clipping at global norm
$1.0$, cosine learning-rate decay, and early stopping on validation RMSE. All
reported numbers are produced by one evaluation pipeline that loads each
checkpoint, evaluates on the same HF-GT test split, and computes the metrics in
Appendix~\ref{app:metrics}. Thus, rows in Table~\ref{tab:baseline_comparison}
differ by architecture, not by data access, optimization, or evaluation code.

\subsection{Baseline families}
\label{app:baseline_families}

\paragraph{Image-to-image regressors.}
We include standard dense-prediction architectures: CNN~\cite{lecun2002gradient},
U-Net~\cite{ronneberger2015u}, ResNet~\cite{he2016resnet}, and
ViT~\cite{dosovitskiy2021vit}. These models treat the task as generic
multi-channel image regression from scene features to an RSS map. Each model
predicts a residual correction over the low-fidelity input, which gives a fair
starting point because the low-fidelity map already contains useful propagation
structure.

\paragraph{Wireless deep-learning baselines.}
We include wireless-specific models that are closest to the task setting:
NeRF$^2$~\cite{zhao2023nerf2}, RadioUNet~\cite{levie2021radiounet},
GeNeRT~\cite{bian2025genert}, and WiGATr~\cite{orekondy2023winert}. These
models are adapted to our 2D indoor RSS-prediction task by replacing their
native input/output heads where needed while preserving their core modeling
principles. They are trained with the same inputs, parameter budget, and tuning
protocol as the other baselines.

\paragraph{Neural-operator baselines.}
We compare against monolithic neural operators: FNO~\cite{li2020fourier},
TFNO~\cite{kossaifi2023multi}, UNO~\cite{rahman2022u},
SFNO~\cite{bonev2023spherical}, CodaNO~\cite{rahman2024pretraining}, and
WNO~\cite{tripura2023wavelet}. These are the closest methodological baselines
because they also learn field-to-field maps. Unlike RayNO, however, they apply
one homogeneous operator family across the whole RSS field rather than separating
specular, diffraction, and scattering refinements.

\begin{revisedblock}
\paragraph{\revised{Classical multi-fidelity regression and per-scene fitting.}}
We also compare against two methods outside the three families above. The first is autoregressive
co-kriging (AR1), the standard non-neural way to combine a cheap and an expensive
simulator~\cite{kennedy2000predicting}. We fit the AR1 multi-fidelity Gaussian process in
Emukit~\cite{paleyes2019emukit} on a low-rank coefficient representation of the radio maps. AR1 assumes one
global relation between the two fidelity levels, whereas the correction from LF to HF changes across the
floorplan with the positions of walls, corners, and materials. The second is GSRF~\cite{yang2025gsrf}, a
wireless Gaussian-splatting method with public code and scalar RSSI support, which we adapt to our
$128\times128$ receive plane. GSRF fits a new set of Gaussian primitives for every scene, while every model in
Table~\ref{tab:baseline_comparison} is trained once and applied to unseen floorplans. We therefore report
both methods separately in Table~\ref{tab:mf-baselines}. Only the metrics shown were computed for them.

\begin{table}[h]
\centering
\small
\caption{\revised{Comparison with classical multi-fidelity regression and per-scene fitting, against the
held-out HF-GT reference.}}
\label{tab:mf-baselines}
\begin{tabular}{lcccc}
\toprule
Method & RMSE (dB) $\downarrow$ & SSIM $\uparrow$ & Fading ratio $\approx 1$ & Tail error $\downarrow$ \\
\midrule
AR1 MF-GP~\cite{kennedy2000predicting,paleyes2019emukit} & 13.046 & 0.646 & 0.308 & 1.423 \\
GSRF~\cite{yang2025gsrf} (fit per scene)                 & 5.070  & 0.876 & 0.729 & 1.157 \\
RayNO (trained once)                                     & \textbf{3.471} & \textbf{0.928} & \textbf{1.013} & \textbf{0.101} \\
\bottomrule
\end{tabular}
\end{table}
\end{revisedblock}

\subsection{Reference: IF-GT labels}
\label{app:partial_gt_row}

The IF-GT row in Table~\ref{tab:baseline_comparison} evaluates the training
label itself against HF-GT, with no learned model in between. This row measures
the residual error of the supervision signal. It is the reference for the
paper's central claim: a learned model that improves over this row is producing
a prediction closer to HF-GT than the labels used for training.

\subsection{Why rankings differ across metric families}
\label{app:why_rankings_differ}

Table~\ref{tab:baseline_comparison} shows that strong image-quality scores do
not necessarily imply strong wireless performance. RMSE, SSIM, and LPIPS are
dominated by the broad, smooth RSS structure, while wireless deployment depends
heavily on low-RSS regions, local fading, and cell-edge behavior. This is why
several baselines appear competitive under image metrics but fail on Outage F1
or fading ratio. The wireless metrics below are included to expose exactly this
failure mode.

\section{Evaluation Metrics}
\label{app:metrics}

This appendix defines the evaluation protocol used for
Table~\ref{tab:baseline_comparison}. Standard pixel and perceptual metrics are
reported for comparability with dense-regression and image-to-image work. The
wireless metrics are defined more explicitly because they measure deployment
properties that generic image metrics can miss.

\subsection{Reference field and aggregation}
\label{app:metric_aggregation}

All metrics are computed against the held-out HF-GT reference
$\widetilde{f}(X)$ from Appendix~\ref{app:dataset}, never against the IF-GT
training labels. Let $\widehat{Y}$ be a model prediction and
$\widetilde{Y}=\widetilde{f}(X)$ be the HF-GT reference, both in dB on the
$128\times128$ receive-plane grid. Metrics are computed only over valid pixels,
\begin{equation}
\mathcal{M} \;=\; \big\{\, p \in \Omega_{\text{rx}} \;:\; \widetilde{Y}(p) \text{ is finite and } \widetilde{Y}(p) > Y_{\min} + \varepsilon \,\big\},
\label{eq:valid_mask}
\end{equation}
where $Y_{\min}=-150$ dB is the simulator floor. Pixel metrics are averaged over
all valid pixels in the test split. Image-level metrics are computed per scene
and then averaged across scenes. The same aggregation is used for every model.

\subsection{Pixel and perceptual metrics}
\label{app:pixel_metrics}

We report MAE, RMSE, and PSNR as standard pointwise reconstruction metrics. MAE
and RMSE are measured in dB. PSNR uses a fixed dynamic range
$R=170$ dB, corresponding to the interval from the simulator floor
$-150$ dB to the ceiling $20$ dB.

\subsection{Perceptual metrics}
\label{app:perceptual_metrics}

We also report SSIM~\cite{wang2004image}, edge-aware SSIM (ESSIM), DISTS
\cite{ding2020image}, LPIPS~\cite{zhang2018unreasonable}, and GradMean. SSIM,
DISTS, and LPIPS are computed after mapping dB-domain fields to $[0,1]$; DISTS
and LPIPS use three-channel replication and the standard PIQ implementations.
ESSIM applies SSIM only near HF-GT edge pixels obtained from Sobel gradients and
a Canny detector~\cite{canny2009computational}. GradMean is the mean spatial
gradient magnitude of the prediction and is interpreted by closeness to the
HF-GT GradMean, not by monotone increase or decrease.

\subsection{Wireless deployment metrics}
\label{app:wireless_metrics}

Wireless deployment decisions depend on where coverage holes occur, how much
local fading exists, and how accurately the cell-edge tail is predicted. These
properties can be hidden by high SSIM or low RMSE, so we report the following
wireless-specific metrics. Constants are listed in
Table~\ref{tab:wireless_constants}.

\subsubsection{Outage F1}
\label{app:outage_f1}

Outage is defined as RSS below $T_{\text{out}}=-100$ dBm. We threshold both the
prediction and HF-GT reference on valid pixels and compute the binary F1 score:
\begin{equation}
\mathrm{F1}_{\text{out}} \;=\; \frac{2 \cdot \mathrm{P} \cdot \mathrm{R}}{\mathrm{P} + \mathrm{R}}, \qquad \mathrm{P} = \frac{\mathrm{TP}}{\mathrm{TP} + \mathrm{FP}}, \qquad \mathrm{R} = \frac{\mathrm{TP}}{\mathrm{TP} + \mathrm{FN}},
\label{eq:outage_f1}
\end{equation}
where TP, FP, and FN are counted over valid pixels. Outage F1 measures whether
the surrogate identifies coverage holes. This is important because a smoothed
prediction can have good RMSE while failing to predict weak-signal regions.

\subsubsection{Fading ratio (LocalStd9)}
\label{app:fading_ratio}

For each valid pixel, we compute the local RSS standard deviation in a
$9\times9$ window. Let $\sigma_9(\widehat{Y})(p)$ and
$\sigma_9(\widetilde{Y})(p)$ denote this quantity for the prediction and
reference. We evaluate fading only in the high-variation subset
\begin{equation}
\mathcal{H} \;=\; \big\{\, p\in\mathcal{M} \;:\; \sigma_9(\widetilde{Y})(p) \geq Q_{0.85}\big(\sigma_9(\widetilde{Y})|_{\mathcal{M}}\big) \,\big\},
\label{eq:high_var_mask}
\end{equation}
and report
\begin{equation}
\rho_{\text{fade}} \;=\; \frac{\mathrm{mean}_{p\in\mathcal{H}}\, \sigma_9(\widehat{Y})(p)}{\mathrm{mean}_{p\in\mathcal{H}}\, \sigma_9(\widetilde{Y})(p)}.
\label{eq:fading_ratio}
\end{equation}
The ideal value is $1$. Values below $1$ indicate oversmoothing, while values
above $1$ indicate excessive local texture.

\subsubsection{SE 5\% tail error}
\label{app:se_tail}

We convert RSS to spectral efficiency using a standard AWGN approximation. With
bandwidth $B=20$ MHz, noise figure $F=7$ dB, and thermal noise
$-174$ dBm/Hz,
\begin{equation}
N_{\text{dBm}} \;=\; -174 + 10\log_{10}(B) + F.
\label{eq:noise_dbm}
\end{equation}
The spectral efficiency at pixel $p$ is
\begin{equation}
\mathrm{SE}(Y)(p) \;=\; \log_2\!\left(1 + 10^{(Y(p) - N_{\text{dBm}})/10}\right) \quad [\mathrm{b/s/Hz}].
\label{eq:se_def}
\end{equation}
On the high-variation mask $\mathcal{H}$, we report the absolute error of the
5th percentile:
\begin{equation}
\big|\mathrm{SE}_{5\%}^{\text{err}}\big| \;=\; \Big| Q_{0.05}\!\big(\mathrm{SE}(\widehat{Y})|_{\mathcal{H}}\big) - Q_{0.05}\!\big(\mathrm{SE}(\widetilde{Y})|_{\mathcal{H}}\big) \Big| \quad [\mathrm{b/s/Hz}].
\label{eq:se_tail}
\end{equation}
This metric measures whether the model preserves the weakest part of the
spectral-efficiency distribution, which is often where deployment decisions are
made.

\subsubsection{MCESE: mean cell-edge SE error}
\label{app:mcese}

Within $\mathcal{H}$, we define the cell-edge subset as the bottom 10\% of
pixels by HF-GT RSS:
\begin{equation}
\mathcal{C} \;=\; \big\{\, p\in\mathcal{H} \;:\; \widetilde{Y}(p) \leq Q_{0.10}\big(\widetilde{Y}|_{\mathcal{H}}\big) \,\big\}.
\label{eq:cell_edge_mask}
\end{equation}
MCESE is the mean absolute spectral-efficiency error on this subset:
\begin{equation}
\mathrm{MCESE} \;=\; \frac{1}{|\mathcal{C}|} \sum_{p\in\mathcal{C}} \big| \mathrm{SE}(\widehat{Y})(p) - \mathrm{SE}(\widetilde{Y})(p) \big| \quad [\mathrm{b/s/Hz}].
\label{eq:mcese}
\end{equation}
This directly measures prediction error at the weakest high-variation pixels,
where coverage and AP-placement decisions are most sensitive.

\subsection{Summary of constants}
\label{app:metric_constants}

\begin{table}[h]
\centering
\small
\caption{Constants used by the wireless deployment metrics. These values are fixed across all reported experiments and are taken from standard 5--6\,GHz indoor wireless practice.}
\label{tab:wireless_constants}
\begin{tabular}{lll}
\toprule
Symbol & Quantity & Value \\
\midrule
$T_{\text{out}}$ & Outage threshold & $-100$ dBm \\
$B$ & Bandwidth & $20$ MHz \\
$F$ & Receiver noise figure & $7$ dB \\
$N_{\text{dBm}}$ & Thermal noise floor (per Eq.~\eqref{eq:noise_dbm}) & $\approx -94$ dBm \\
$Y_{\min}$ & Simulator floor for valid mask & $-150$ dB \\
$R$ & Dynamic range for PSNR & $170$ dB \\
$w$ & Local-std window size & $9\times 9$ \\
$q_{\mathrm{HV}}$ & High-variance quantile cutoff & $0.85$ \\
$q_{\mathrm{CE}}$ & Cell-edge quantile cutoff & $0.10$ \\
$q_{\mathrm{tail}}$ & SE-tail quantile & $0.05$ \\
\bottomrule
\end{tabular}
\end{table}

\begin{revisedblock}
\subsection{\revised{What the fading ratio does and does not measure}}
\label{app:fading-scope}

A question that arises when comparing these metrics against a real survey is what an RSS map actually contains. Reflection, diffraction, and scattering are not separate quantities recoverable from each RSS value; they are the mechanisms that shape the map, appearing in a fixed-frequency spatial survey as the broad coverage pattern, the sharp changes near corners and shadow boundaries, and the finer variation between nearby locations. The fading ratio we report measures that spatial variation, not temporal fading from mobility or Doppler, so a survey compared against the model must use comparable time and bandwidth averaging; heavier averaging suppresses precisely the variation the wireless metrics are built to detect.

\end{revisedblock}

\subsection{Why we report all three families}
\label{app:metric_families}

Pixel, perceptual, and wireless metrics answer different questions. Pixel
metrics measure average numerical reconstruction error. Perceptual metrics
measure structural similarity when RSS fields are treated as images. Wireless
metrics measure whether the prediction supports deployment decisions such as
coverage planning, fading-margin estimation, and cell-edge performance
assessment. We therefore report all three, but place the most emphasis on the
wireless metrics when discussing deployment relevance.

\begin{figure}[t]
    \centering
    \begin{subfigure}{\linewidth}
        \centering
        \includegraphics[width=\linewidth]{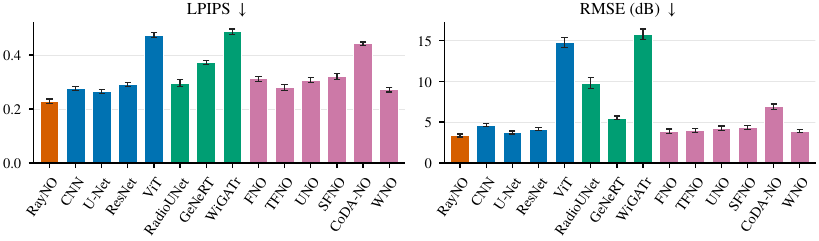}
        \caption{Pixel-level and perceptual metrics (LPIPS, RMSE).}   %
    \end{subfigure}
    \vspace{0.2cm}   %
    \begin{subfigure}{\linewidth}
        \centering
        \includegraphics[width=\linewidth]{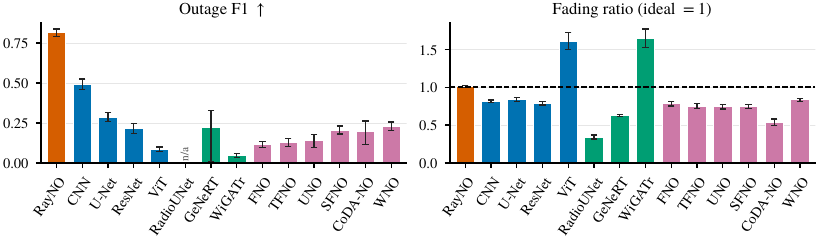}
        \caption{Deployment-critical wireless metrics (Outage F1, fading ratio).}
    \end{subfigure}
    \caption{Performance comparison across generic and wireless-specific metrics with 95\% bootstrap confidence interval.}
    \label{fig:error_bars}
\end{figure}

\section{Extended Results}
\label{app:extended-results}

Table~\ref{tab:extended_full_metrics} gives the full metric version of the main
comparison against the held-out HF-GT reference. It extends
Table~\ref{tab:baseline_comparison} by adding outage recall, PSNR, ESSIM,
DISTS, and GradMean. GradMean is included as a global roughness diagnostic: the
HF-GT reference has GradMean $1.58$, so this column should be interpreted by
closeness to $1.58$, not by monotone increase or decrease.

\begin{table}[t]
\centering
\scriptsize
\setlength{\tabcolsep}{2.2pt}
\caption{
Extended quantitative comparison on the HF-GT indoor test set.
R denotes outage recall. 
Fad. is the local fading ratio and should be close to $1$.
SE5 is the absolute 5th-percentile spectral-efficiency tail error in high-variation regions.
Grad is the mean spatial RSS-gradient magnitude. The HF-GT reference has GradMean $1.58$; hence Grad should be interpreted by closeness to $1.58$, not by monotone increase or decrease.
Bold marks the best learned predictor.
}
\label{tab:extended_full_metrics}
\resizebox{\textwidth}{!}{
\begin{tabular}{llccccccccccccc}
\toprule
Family & Model 
& F1$\uparrow$ & R$\uparrow$ 
& Fad.$\approx1$ & SE5$\downarrow$ & MCESE$\downarrow$
& RMSE$\downarrow$ & MAE$\downarrow$ & PSNR$\uparrow$ 
& SSIM$\uparrow$ & ESSIM$\uparrow$ & LPIPS$\downarrow$ & DISTS$\downarrow$ & Grad$\approx1.58$ \\
\midrule
Label & IF-GT 
& 0.172 & 0.991 & 1.782 & 1.661 & 1.311 
& 14.012 & 5.486 & 21.679 & 0.725 & 0.405 & 0.476 & 0.294 & 9.348 \\
\midrule
CV & CNN 
& 0.490 & 0.348 & 0.799 & 1.043 & 1.443 
& 4.798 & 3.152 & 30.989 & 0.904 & 0.548 & 0.277 & 0.202 & 1.380 \\
CV & U-Net 
& 0.178 & 0.100 & 0.815 & 0.448 & 1.076 
& 3.851 & 2.326 & 32.897 & 0.919 & 0.551 & 0.266 & 0.199 & 1.388 \\
CV & ResNet 
& 0.107 & 0.058 & 0.761 & 0.531 & 1.152 
& 4.228 & 2.692 & 32.085 & 0.905 & 0.509 & 0.291 & 0.211 & 1.368 \\
CV & ViT 
& 0.017 & 0.045 & 1.495 & 0.145 & 1.637 
& 15.310 & 6.188 & 20.909 & 0.652 & 0.157 & 0.475 & 0.336 & 10.313 \\
\midrule
Wireless & NeRF$^2$
& 0.026 & 0.015 & 0.696 & 1.725 & 3.406
& 5.671 & 3.041 & 27.291 & 0.857 & 0.399 & 0.345 & 0.203 & 1.935 \\
Wireless & RadioUNet 
& 0.000 & 0.000 & 0.314 & 9.356 & 9.348 
& 10.327 & 5.960 & 24.329 & 0.840 & 0.325 & 0.296 & 0.268 & 0.432 \\
Wireless & GeNeRT 
& 0.000 & 0.000 & 0.608 & 1.405 & 2.133 
& 5.620 & 3.932 & 29.615 & 0.857 & 0.401 & 0.373 & 0.247 & 1.530 \\
Wireless & WiGATr 
& 0.002 & 0.005 & 1.517 & 0.559 & 2.078 
& 16.305 & 6.807 & 20.362 & 0.643 & 0.132 & 0.487 & 0.347 & 10.282 \\
\midrule
NO & FNO 
& 0.025 & 0.013 & 0.750 & 0.690 & 1.605 
& 4.045 & 2.362 & 32.472 & 0.903 & 0.443 & 0.313 & 0.225 & 1.406 \\
NO & TFNO 
& 0.018 & 0.009 & 0.715 & 0.805 & 1.848 
& 4.125 & 2.280 & 32.300 & 0.912 & 0.439 & 0.280 & 0.217 & 1.138 \\
NO & UNO 
& 0.030 & 0.015 & 0.710 & 0.592 & 1.671 
& 4.384 & 2.631 & 31.771 & 0.898 & 0.415 & 0.307 & 0.228 & 1.256 \\
NO & SFNO 
& 0.111 & 0.060 & 0.719 & 0.556 & 1.645 
& 4.467 & 2.633 & 31.608 & 0.891 & 0.384 & 0.320 & 0.234 & 1.228 \\
NO & CoDA-NO 
& 0.040 & 0.029 & 0.501 & 2.080 & 4.320 
& 7.153 & 4.776 & 27.519 & 0.775 & 0.204 & 0.442 & 0.283 & 1.848 \\
NO & WNO 
& 0.191 & 0.108 & 0.809 & 0.373 & 0.965 
& 3.982 & 2.414 & 32.606 & 0.913 & 0.552 & 0.272 & 0.198 & 1.413 \\
\midrule
Ours & RayNO 
& \textbf{0.820} & \textbf{0.852} & \textbf{1.013} & \textbf{0.100} & \textbf{0.747} 
& \textbf{3.471} & \textbf{2.086} & \textbf{33.800} & \textbf{0.928} & \textbf{0.576} & \textbf{0.229} & \textbf{0.179} & \textbf{1.582} \\
\bottomrule
\end{tabular}
}
\end{table}

The additional columns reinforce the main-text conclusion without changing it:
RayNO is the best learned predictor across the reported metrics, and its
GradMean of $1.582$ closely matches the HF-GT value of $1.58$. By contrast,
IF-GT has GradMean $9.348$, indicating that the training labels contain much
more high-frequency Monte Carlo roughness than the HF-GT reference. This supports
the central claim that RayNO does not simply copy the IF-GT label texture.

Figure~\ref{fig:pred_comparison} provides qualitative examples of the same
behavior, and Figure~\ref{fig:error_bars} reports bootstrap uncertainty for the
main metric families.

\subsection{Downstream Case Study: Access-Point Placement}
\label{app:ap-placement}

This case study translates the wireless metrics into a simple AP-placement
interpretation. It is not a separate benchmark; it is an illustrative
calculation using the metrics already reported in Table~\ref{tab:extended_full_metrics}.
We consider a $10{,}000~\mathrm{m}^2$ enterprise indoor floorplan at $5.5$ GHz
and compare RayNO with one representative baseline from each family: CNN,
GeNeRT, and WNO.

\paragraph{Coverage holes and extra APs.}
Assume that $10\%$ of the floorplan lies in hard-coverage regions that must be
detected before final AP placement. Thus,
\[
A_{\mathrm{out}} = 0.10 \times 10{,}000 = 1{,}000~\mathrm{m}^2.
\]
Using outage recall, the missed area is
\[
A_{\mathrm{miss}} = (1-\mathrm{Recall})A_{\mathrm{out}}.
\]
Assume one additional enterprise AP can correct approximately
$250~\mathrm{m}^2$ of missed hard-coverage area and that the installed retrofit
cost is \$1,500 per AP:
\[
N_{\mathrm{retrofit}}=\left\lceil A_{\mathrm{miss}}/250 \right\rceil .
\]

\begin{table}[t]
\centering
\scriptsize
\setlength{\tabcolsep}{5pt}
\caption{
Engineering interpretation of outage detection for AP placement on a $10{,}000~\mathrm{m}^2$ enterprise floorplan. 
The assumed true hard-coverage area is $1{,}000~\mathrm{m}^2$.
}
\label{tab:ap_coverage_case}
\begin{tabular}{lccccc}
\toprule
Model & Outage F1 $\uparrow$ & Recall $\uparrow$ & Missed area $(\mathrm{m}^2)$ & Extra APs & Extra cost \\
\midrule
CNN & 0.490 & 0.348 & $652$ & $3$ & \$4,500 \\
GeNeRT & 0.000 & 0.000 & $1{,}000$ & $4$ & \$6,000 \\
WNO & 0.191 & 0.108 & $892$ & $4$ & \$6,000 \\
RayNO & \textbf{0.820} & \textbf{0.852} & \textbf{148} & \textbf{1} & \textbf{\$1,500} \\
\bottomrule
\end{tabular}
\end{table}

Under these assumptions, RayNO leaves $148~\mathrm{m}^2$ of missed hard-coverage
area, corresponding to one retrofit AP. CNN leaves $652~\mathrm{m}^2$, while
GeNeRT and WNO leave $1{,}000~\mathrm{m}^2$ and $892~\mathrm{m}^2$, respectively.

\paragraph{Achieved cell-edge throughput.}
For an $80$ MHz enterprise WiFi channel, an error of $1$ bit/s/Hz corresponds
to an $80$ Mbps throughput loss. With a $100$ Mbps cell-edge target, we use
\[
R_{\mathrm{achieved}}
=
\max\left(0,\;100 - 80 \times \mathrm{SE5}\right)\ \mathrm{Mbps}.
\]

\begin{table}[t]
\centering
\scriptsize
\setlength{\tabcolsep}{6pt}
\caption{
Engineering interpretation of 5\% tail spectral-efficiency error as achieved cell-edge throughput. 
The target edge throughput is $100$ Mbps and the assumed bandwidth is $80$ MHz.
}
\label{tab:ap_throughput_case}
\begin{tabular}{lccc}
\toprule
Model & SE5 error $\downarrow$ & Throughput loss & Achieved edge throughput $\uparrow$ \\
\midrule
CNN & 1.043 & $83.4$ Mbps & $16.6$ Mbps \\
GeNeRT & 1.405 & $112.4$ Mbps & $0.0$ Mbps \\
WNO & 0.373 & $29.8$ Mbps & $70.2$ Mbps \\
RayNO & \textbf{0.100} & \textbf{8.0} Mbps & \textbf{92.0} Mbps \\
\bottomrule
\end{tabular}
\end{table}

RayNO achieves an estimated cell-edge throughput of $92.0$ Mbps, compared with
$16.6$ Mbps for CNN, $0.0$ Mbps for GeNeRT, and $70.2$ Mbps for WNO.

\paragraph{Fading and latency risk.}
The fading ratio indicates whether the model preserves local multipath
fluctuation strength. RayNO has fading ratio $1.013$, close to the ideal value
of $1$. CNN, GeNeRT, and WNO have fading ratios $0.799$, $0.608$, and $0.809$,
respectively. Thus, these baselines under-estimate fading by roughly
$19$--$39\%$, while RayNO over-estimates it by only $1.3\%$. For AP planning,
this means RayNO is less likely to make locally unstable regions appear
artificially safe.
\begin{revisedblock}
\subsection{\revised{Robustness to material-label error}}
\label{app:material-robustness}

Ablation 3B in Table~\ref{tab:ablation_study} removes the material map entirely and shows the information is
necessary. It does not show how much labeling error the model tolerates, which is the question that matters
when the floorplan comes from a building survey rather than from the generator. We therefore took the full
trained checkpoint, left it unchanged, and corrupted a fraction of the per-pixel material labels at inference
only, reassigning each corrupted pixel to a different class from the indoor alphabet.

\begin{table}[h]
\centering
\small
\caption{\revised{Material-label robustness. The full trained checkpoint is evaluated unchanged while a fraction of the per-pixel material labels is corrupted at inference. ``F1 kept'' is Outage F1 relative to the uncorrupted run.}}
\label{tab:material_corruption}
\begin{tabular}{ccccc}
\toprule
\revised{Corrupted labels} & \revised{Outage F1 $\uparrow$} & \revised{F1 kept} & \revised{RMSE (dB) $\downarrow$} & \revised{SSIM $\uparrow$} \\
\midrule
\revised{0\%}  & \revised{0.820} & \revised{100.0\%} & \revised{3.471} & \revised{0.928} \\
\revised{10\%} & \revised{0.746} & \revised{90.9\%}  & \revised{3.459} & \revised{0.927} \\
\revised{20\%} & \revised{0.609} & \revised{74.3\%}  & \revised{3.592} & \revised{0.925} \\
\revised{50\%} & \revised{0.062} & \revised{7.5\%}   & \revised{4.376} & \revised{0.911} \\
\bottomrule
\end{tabular}
\end{table}

The overall map is stable throughout: RMSE moves by less than $0.15$\,dB up to $20\%$ corruption and SSIM by
$0.003$. Outage F1 decays faster because it is a thresholded quantity, so a small local error near a coverage
boundary moves pixels across the outage decision while barely changing the average error. Performance breaks
down only once roughly half the material map is wrong, at which point the input no longer describes the same
building. Because the model was never retrained under corruption, these numbers are a conservative estimate
of what retraining with label noise would achieve.
\end{revisedblock}

\begin{revisedblock}
\subsection{\revised{Averaging intermediate-fidelity runs: the cost--quality curve}}
\label{app:if-averaging}

The simplest non-neural way to reduce Monte Carlo noise is to run the intermediate-fidelity simulation
several times with different seeds and average the results. Averaging only helps if the noise is zero-mean,
so this experiment also checks the noise model. Every improvement along this curve, however, requires
running the simulator again.

We generated independent $10^{6}$-ray Stage-3 maps for a fixed scene by changing only the Sionna seed,
averaged them in linear power, and measured quality against the HF-GT reference.
Table~\ref{tab:if-averaging} reports the trend.

\begin{table}[h]
\centering
\small
\caption{\revised{Quality and cost of averaging $N$ independent IF realizations of the same scene, against
the held-out HF-GT reference. Runtime is simulator time; the final column is relative to the $28$\,s HF-GT
simulation. \gls{sysname} is shown for comparison and needs one LF simulation plus a $12.1$\,ms forward pass.}}
\label{tab:if-averaging}
\begin{tabular}{lcccc}
\toprule
Method & RMSE (dB) $\downarrow$ & Outage F1 $\uparrow$ & Runtime & Runtime / HF-GT \\
\midrule
IF mean, $N=1$    & $7.910\pm0.112$ & $0.371\pm0.009$ & $1.8$\,s  & $0.064\times$ \\
IF mean, $N=2$    & $6.177\pm0.084$ & $0.471\pm0.010$ & $3.6$\,s  & $0.129\times$ \\
IF mean, $N=5$    & $4.903\pm0.072$ & $0.583\pm0.012$ & $9.0$\,s  & $0.321\times$ \\
IF mean, $N=10$   & $4.308\pm0.066$ & $0.637\pm0.013$ & $18.0$\,s & $0.643\times$ \\
IF mean, $N=20$   & $3.890\pm0.056$ & $0.667\pm0.013$ & $36.0$\,s & $1.286\times$ \\
\midrule
\gls{sysname}     & $\mathbf{3.471}$ & $\mathbf{0.820}$ & $\mathbf{0.112}$\,s & $\mathbf{0.004\times}$ \\
\bottomrule
\end{tabular}
\end{table}

The improvement is monotonic, which is what an unbiased estimator should do and is consistent with the
repeated-seed evidence in Appendix~\ref{app:seed-validation}. The cost, however, grows linearly. By $N=20$ the
averaging has spent $36$\,s, more than the $28$\,s a single HF-GT simulation would have taken, and it is still
behind \gls{sysname} on both RMSE ($3.890$ against $3.471$\,dB) and Outage F1 ($0.667$ against $0.820$).
Averaging buys accuracy by re-paying the simulator cost that the surrogate is designed to remove.
\end{revisedblock}

\section{Theoretical appendix: proofs for \revised{denoising without clean references}}
\label{app:theory}

This appendix provides the formal statement, proof, and supporting results for
the \revised{denoising-without-clean-references} claim stated informally in Section~\ref{sec:method},
Eq.~(2). After the roadmap (Section~\ref{app:roadmap}), we introduce setup and
notation (Section~\ref{app:setup}), state the assumptions
(Section~\ref{app:assumptions}), justify conditional unbiasedness of finite-ray
Monte Carlo ray tracing (Section~\ref{app:mcrt}), prove the population
decomposition (Section~\ref{app:population}), and prove
Theorem~\ref{thm:zsd} (Section~\ref{app:mainproof}). We then include two
extensions: a finite-dimensional realizable analogue
(Section~\ref{app:ols}) and a Sobolev target-preservation result supporting
the training objective in Section~\ref{sec:training-obj}
(Section~\ref{app:sobolev}).

\subsection{Roadmap}
\label{app:roadmap}

Theorem~\ref{thm:zsd} formalizes when empirical risk minimization on noisy
IF-GT labels can recover the clean field and eventually become closer to the
high-fidelity reference than the labels themselves. Proposition~\ref{prop:mcrt}
connects the noise model to finite-budget ray tracing, and
Lemma~\ref{lem:population} gives the key risk-decomposition identity.
Proposition~\ref{prop:ols} provides a simple realizable analogue, while
Proposition~\ref{prop:sobolev} and Corollary~\ref{cor:sobolev-weighted} extend
the same argument to the discrete Sobolev and weighted Sobolev losses used in
Section~\ref{sec:training-obj}.

\subsection{Setup and notation}
\label{app:setup}

We consider an indoor wireless scene with one transmitter and many
possible receiver locations on a two-dimensional receive plane
$\Omega \subset \mathbb{R}^2$. The quantity of interest is the received
signal strength (RSS) field,
\[
Y : \Omega \to \mathbb{R},
\]
where $Y(p)$ is the received power, in dB, at receiver location $p \in
\Omega$. We predict this scalar power field only; we do not model phase,
channel impulse responses, angular spectra, or MIMO channel matrices.

The physical difficulty is that the RSS field is not determined only by
distance from the transmitter. Walls, corners, doorways, furniture, and
materials create multiple propagation paths. Some paths travel directly
or reflect from large surfaces, producing broad coverage trends. Other
paths bend around corners or shadow boundaries, producing sharp
transitions. Small objects and material changes add local fluctuations.
RayNO is built around this simple decomposition: global transport,
edge-driven correction, and local scattering correction.

A single input instance is denoted
\[
X = (u, g, c) \in \mathcal{X}.
\]
Here $u$ is a cheap low-fidelity RSS field from ray tracing, used as a
physical scaffold rather than as the training label. The term $g$
denotes geometry and material information aligned with the receive
plane, such as wall and obstacle maps, distance-to-transmitter maps,
shadow or line-of-sight cues, and material labels. The term $c$ denotes
coordinate information, including receiver-grid coordinates and
transmitter/receiver height information. The output space $\mathcal{Y}$
is the space of RSS fields on $\Omega$ equipped with the norm used in
the analysis below.

Let
\[
f^\star : \mathcal{X} \to \mathcal{Y}
\]
denote the clean field operator: for an input scene $X$, $f^\star(X)$ is
the ideal RSS field that would be obtained by averaging over the
simulator's ray-sampling randomness. This clean field is the object we
would like to recover, but it is not observed directly.

Training uses a finite-budget ray-tracing label
\[
Y = f^\star(X) + \varepsilon,
\]
where $\varepsilon$ is the simulation noise caused by using a finite
number of sampled rays. Evaluation uses a much higher-budget reference
\[
\tilde f(X) = f^\star(X) + \delta_{\mathrm{abs}},
\]
where $\delta_{\mathrm{abs}}$ is the remaining error of the high-fidelity
reference relative to the clean simulator limit. The key asymmetry is
that the training label is noisier than the evaluation reference.
Intuitively, the IF-GT label is a noisy photograph of the
field, while the high-fidelity reference is a much cleaner photograph of
the same underlying object.

We learn an operator $f_\theta : \mathcal{X}\to\mathcal{Y}$ from
independent training examples $(X_i,Y_i)_{i=1}^n$ by empirical risk
minimization:
\[
\hat f_n \in
\arg\min_{f\in\mathcal{F}}
\frac{1}{n}\sum_{i=1}^{n}
\ell_f(X_i,Y_i),
\qquad
\ell_f(X,Y) := \|f(X)-Y\|_{\mathcal{Y}}^2 .
\]
Theorem~\ref{thm:zsd} below explains when this procedure can produce a
predictor closer to the high-fidelity reference $\tilde f(X)$ than the
training label $Y$ itself. The informal reason is simple: if the
simulation noise is conditionally zero-mean, then fitting many noisy
labels can recover their average field rather than their sample-specific
noise.

\paragraph{Notation summary.}
Table~\ref{tab:notation-app} consolidates the symbols used throughout
the appendix. We follow the convention that hatted quantities
($\hat{f}$, $\hat{Y}$) denote model predictions, tilded quantities
($\tilde{f}$) denote High-fidelity references, and starred quantities
($f^{\star}$) denote the unobservable clean field. Two expectation
symbols appear in the proofs: $\mathbb{E}[\,\cdot\,]$ denotes expectation
over a fresh test point $(X,Y)$, while $\mathbb{E}_S[\,\cdot\,]$ denotes
expectation over the i.i.d.\ training sample $S=(X_i,Y_i)_{i=1}^n$.

\begin{table}[h]
\centering
\footnotesize
\caption{Comprehensive notation. Symbols are used consistently across
this appendix.}
\label{tab:notation-app}
\setlength{\tabcolsep}{6pt}
\renewcommand{\arraystretch}{1.05}
\begin{tabular}{ll}
\toprule
\textbf{Symbol} & \textbf{Meaning} \\
\midrule
\multicolumn{2}{l}{\emph{Domains and spaces}} \\
$\Omega \subset \mathbb{R}^2$ & 2D floorplan/deployment domain \\
$\mathcal{Y}$ & Separable Hilbert space of RSS fields on $\Omega$ \\
$\langle\cdot,\cdot\rangle_\mathcal{Y}$, $\|\cdot\|_\mathcal{Y}$ & Inner product and norm on $\mathcal{Y}$ \\
$\mathcal{G}$ & Space of scene encodings (geometry, materials) on $\Omega$ \\
$\mathcal{C}$ & Space of coordinate priors (Rx grid, Tx/Rx elevations) \\
$\mathcal{X} = \mathcal{Y} \times \mathcal{G} \times \mathcal{C}$ & Input space \\
$X = (u, g, c) \in \mathcal{X}$ & Input instance, with law $P_X$ \\
$u \in \mathcal{Y}$ & LF input ($10^4$-ray, single-bounce specular) \\
$g \in \mathcal{G}$ & Scene encoding (walls, materials, obstacles) \\
$c \in \mathcal{C}$ & Coordinate priors (Rx grid; Tx/Rx elevations) \\
\midrule
\multicolumn{2}{l}{\emph{Fields and labels}} \\
$f^\star : \mathcal{X} \to \mathcal{Y}$ & Clean field operator (population limit of MC RT) \\
$Y \in \mathcal{Y}$ & IF-GT label ($10^6$-ray, full physics) \\
$\tilde f(X) \in \mathcal{Y}$ & HF-GT reference ($10^8$-ray, full physics) \\
$\varepsilon = Y - f^\star(X)$ & MC noise; $\mathbb{E}[\varepsilon \mid X]=0$ \\
$\delta_{\mathrm{abs}}(X) = \tilde f(X) - f^\star(X)$ & Reference residual \\
$\sigma_{\mathrm{MC}}^2 = \mathbb{E}\|\varepsilon\|_\mathcal{Y}^2$ & MC noise floor \\
$\eta_{\mathrm{abs}}^2 = \mathbb{E}\|\delta_{\mathrm{abs}}(X)\|_\mathcal{Y}^2$ & Reference variance \\
\midrule
\multicolumn{2}{l}{\emph{Learning}} \\
$\mathcal{F}$ & Hypothesis class (a set of operators $\mathcal{X} \to \mathcal{Y}$) \\
$f_\theta \in \mathcal{F}$ & Parameterized neural operator \\
$\hat f_n$ & ERM solution from $n$ i.i.d.\ samples \\
$\ell_f(x,y) = \|f(x) - y\|_\mathcal{Y}^2$ & Squared loss \\
$\mathcal{L}(f) = \mathbb{E}\,\ell_f(X,Y)$ & Population risk \\
$\widehat{\mathcal{L}}_n(f) = \tfrac1n\sum_{i=1}^n \ell_f(X_i,Y_i)$ & Empirical risk \\
$\mathbb{E}_S[\,\cdot\,]$ & Expectation over training sample $S=(X_i,Y_i)_{i=1}^n$ \\
$\varepsilon_{\mathrm{approx}}^2 = \inf_{f \in \mathcal{F}}\mathbb{E}\|f-f^\star\|_{L^2(P_X;\mathcal{Y})}^2$ & Approximation error \\
$\mathcal{L}_\mathcal{F} = \{\ell_f : f \in \mathcal{F}\}$ & Loss class \\
$\mathfrak{R}_n(\mathcal{L}_\mathcal{F})$ & Rademacher complexity of $\mathcal{L}_\mathcal{F}$ \\
$C_\mathcal{F}, M$ & Complexity and boundedness constants \\
\midrule
\multicolumn{2}{l}{\emph{Architecture (Section~\ref{sec:method})}} \\
$\tilde{\mathcal{H}}_k$ ($k=1,2,3$) & Stage operator (specular / diffraction / scattering) \\
$x_k, y_k = P_k(x_k)$ & Stage-$k$ latent state and field readout \\
$Y_k$ ($k=1,2,3$) & Stage label; $Y_3 = Y$ \\
$E$ & Shared encoder; $x_0 = E(u, c_g, c_s)$ \\
$c_g \in \mathcal{C}_g$, $c_s \in \mathcal{C}_s$ & Geometry / semantic context (derived from $g$) \\
$s^\star \in [0,1]^\Omega$ & Adaptive sizing field \\
$p \in \Omega$ & Spatial coordinate \\
\bottomrule
\end{tabular}
\end{table}

For any measurable $h:\mathcal{X}\to\mathcal{Y}$, we write
$\|h\|_{L^2(P_X;\mathcal{Y})}^2 := \mathbb{E}_X\|h(X)\|_\mathcal{Y}^2$.

\paragraph{Bridge to main-paper notation.}
For accessibility, the main paper writes the ideal field as
$Y_{\mathrm{oracle}}^{(i)}$, the IF-GT training label as
$Y_{\mathrm{IF\text{-}GT}}^{(i)}$, and the held-out high-fidelity
reference as $Y_{\mathrm{HF\text{-}GT}}^{(i)}$. The translation to the
operator-learning notation used in this appendix is:
\[
f^\star(X_i) \;\equiv\; Y_{\mathrm{oracle}}^{(i)},
\qquad
Y_i \;\equiv\; Y_{\mathrm{IF\text{-}GT}}^{(i)},
\qquad
\tilde f(X_i) \;\equiv\; Y_{\mathrm{HF\text{-}GT}}^{(i)}.
\]
The conditional unbiasedness assumption stated in
Eq.~(\ref{sec:method}) of the main paper,
$\mathbb{E}[\varepsilon_i \mid X_i] = 0$, is exactly Assumption~(A1)
below. The training objective in Eq.~(1) of the main paper, with
$\mathcal{L}$ taken to be squared loss, is the empirical risk
$\widehat{\mathcal{L}}_n(f)$ defined here. We use the operator-learning
notation throughout the appendix because it makes the noise-averaging
mechanism algebraically transparent.

\subsection{Standing assumptions}
\label{app:assumptions}

We begin with the following set of assumptions:
\begin{description}\setlength\itemsep{2pt}
\item[\rm(A1)] \emph{Noise model.} $Y = f^\star(X) + \varepsilon$ with
$\mathbb{E}[\varepsilon \mid X] = 0$ and
$\sigma_{\mathrm{MC}}^2 := \mathbb{E}\|\varepsilon\|_\mathcal{Y}^2 < \infty$.
\item[\rm(A2)] \emph{Reference model.}
$\tilde f(X) = f^\star(X) + \delta_{\mathrm{abs}}(X)$ with
$\eta_{\mathrm{abs}}^2 := \mathbb{E}\|\delta_{\mathrm{abs}}(X)\|_\mathcal{Y}^2 < \infty$.
\item[\rm(A3)] \emph{Bounded loss class.} There exists $M < \infty$ such
that $0 \le \ell_f(x,y) \le M$ for all $f \in \mathcal{F}$ and almost
every $(x,y)$.
\item[\rm(A4)] \emph{Polynomial Rademacher complexity.}
$\mathfrak{R}_n(\mathcal{L}_\mathcal{F}) \le C_\mathcal{F}/\sqrt{n}$ for
some $C_\mathcal{F} < \infty$.
\end{description}
Assumption~(A1) is justified physically and mathematically in
Section~\ref{app:mcrt} below.
Assumption~(A3) is a standard truncation-of-loss assumption; in practice
RSS values lie in a bounded range (we clip to $[-150, 20]$~dB),
which combined with bounded-weight neural operators ensures the loss is
uniformly bounded. Assumption~(A4) holds for any neural operator class
of bounded depth, width, and weight norm with Lipschitz activations by
standard norm-based covering arguments~\cite{bartlett2017spectrally};
specifically for FNO-style architectures with bounded spectral
truncation, $C_\mathcal{F}$ is polynomial in the architectural
constants~\cite{kovachki2021universal}.

\subsection{MC RT as a Hilbert-valued unbiased estimator}
\label{app:mcrt}

We first verify that the noise model~(A1) is mathematically consistent
with finite-budget MC RT as used to generate IF-GT labels.

\begin{proposition}[Hilbert-valued unbiased MC estimator]
\label{prop:mcrt}
Fix $X = x \in \mathcal{X}$. Let $(\Xi, \nu_x)$ denote the path space
of propagation trajectories at scene $x$ under the simulator's physics
setting, and let $G_x : \Xi \to \mathcal{Y}$ map a path $\xi$ to its
field contribution. Assume that $f^\star$ is the Bochner integral
\begin{equation}
\label{eq:mcrt-clean}
f^\star(x) = \int_\Xi G_x(\xi)\,d\nu_x(\xi).
\end{equation}
Let $q_x$ be a sampling distribution with $\nu_x \ll q_x$ and define
$Z_x(\xi) := \tfrac{d\nu_x}{dq_x}(\xi)\,G_x(\xi)$. Suppose
$\mathbb{E}_{q_x}\|Z_x(\xi)\|_\mathcal{Y}^2 < \infty$. Given i.i.d.\
samples $\xi_1,\dots,\xi_m \sim q_x$, the estimator
$Y_m(x) := \tfrac1m \sum_{r=1}^m Z_x(\xi_r)$ satisfies
\begin{equation}
\mathbb{E}[Y_m(x) \mid X = x] = f^\star(x),
\quad
\mathbb{E}[\|Y_m(x) - f^\star(x)\|_\mathcal{Y}^2 \mid X = x]
= \tfrac{1}{m}\,\mathbb{E}_{q_x}\|Z_x(\xi) - f^\star(x)\|_\mathcal{Y}^2.
\end{equation}
\end{proposition}

\begin{proof}
By Radon--Nikodym and Bochner integrability,
$\mathbb{E}[Z_x(\xi) \mid X=x]
= \int_\Xi \tfrac{d\nu_x}{dq_x}(\xi)\,G_x(\xi)\,dq_x(\xi)
= \int_\Xi G_x(\xi)\,d\nu_x(\xi) = f^\star(x)$,
giving the first claim. For the second, expand
$\|Y_m(x) - f^\star(x)\|_\mathcal{Y}^2$ as a double sum and use that
distinct samples $\xi_r, \xi_s$ are conditionally independent and
centered, so cross-terms vanish:
\begin{equation}
\mathbb{E}[\|Y_m(x) - f^\star(x)\|_\mathcal{Y}^2 \mid X=x]
= \tfrac{1}{m^2}\sum_{r=1}^m \mathbb{E}[\|Z_x(\xi_r) - f^\star(x)\|_\mathcal{Y}^2 \mid X=x]
= \tfrac{1}{m}\,\mathbb{E}_{q_x}\|Z_x(\xi) - f^\star(x)\|_\mathcal{Y}^2. \qedhere
\end{equation}
\end{proof}

\begin{remark}
\label{rem:mcrt-noise}
Setting $\varepsilon := Y_m(X) - f^\star(X)$ recovers Assumption~(A1)
with
$\sigma_{\mathrm{MC}}^2 \le \tfrac{1}{m}\sup_x \mathbb{E}_{q_x}\|Z_x - f^\star(x)\|_\mathcal{Y}^2$.
The variance therefore decays as $\mathcal{O}(m^{-1})$ in the ray
budget. Choosing $m = 10^6$ for IF-GT and $m = 10^8$ for HF-GT
gives $\eta_{\mathrm{abs}}^2 / \sigma_{\mathrm{MC}}^2 \approx 10^{-2}$,
consistent with the regime $\eta_{\mathrm{abs}}^2 \ll \sigma_{\mathrm{MC}}^2$
used in the main text.
\end{remark}

\subsection{Population decomposition}
\label{app:population}

The next lemma is the basic identity behind Theorem~\ref{thm:zsd}: under
unbiased MC supervision, the noisy population risk equals the
clean-field risk plus an additive constant.

\begin{lemma}[Population decomposition]
\label{lem:population}
Under~(A1), for every measurable $f : \mathcal{X} \to \mathcal{Y}$,
\begin{equation}
\label{eq:pop-decomp}
\mathcal{L}(f)
= \|f - f^\star\|_{L^2(P_X;\mathcal{Y})}^2 + \sigma_{\mathrm{MC}}^2.
\end{equation}
Hence the population minimizer over all measurable $f$ is exactly
$f^\star$, and the minimizer over $\mathcal{F}$ is the
$L^2(P_X;\mathcal{Y})$-projection of $f^\star$ onto $\mathcal{F}$.
\end{lemma}

\begin{proof}
Decompose $f(X) - Y = (f(X) - f^\star(X)) - \varepsilon$ and expand the
squared norm:
\begin{equation}
\|f(X) - Y\|_\mathcal{Y}^2
= \|f(X) - f^\star(X)\|_\mathcal{Y}^2
- 2\langle f(X) - f^\star(X),\,\varepsilon\rangle_\mathcal{Y}
+ \|\varepsilon\|_\mathcal{Y}^2.
\end{equation}
The cross-term vanishes in expectation by the tower property and~(A1):
$\mathbb{E}\langle f(X) - f^\star(X), \varepsilon\rangle_\mathcal{Y}
= \mathbb{E}\langle f(X) - f^\star(X), \mathbb{E}[\varepsilon \mid X]\rangle_\mathcal{Y} = 0$.
Taking expectations of the remaining two terms yields~\eqref{eq:pop-decomp}.
\end{proof}

\subsection{\revised{Denoising without clean references}}
\label{app:mainproof}

We now state and prove the main result.

\begin{theorem}[\revised{Denoising without clean references}]
\label{thm:zsd}
Assume~(A1)--(A4). Let $\hat f_n$ be the empirical risk minimizer over
$\mathcal{F}$ from $n$ i.i.d.\ training samples $(X_i, Y_i)$, and let
$\varepsilon_{\mathrm{approx}}^2 :=
\inf_{f \in \mathcal{F}}\|f - f^\star\|_{L^2(P_X;\mathcal{Y})}^2$ denote
the approximation error of $\mathcal{F}$. Then:
\begin{enumerate}\setlength\itemsep{2pt}
\item[(i)] \emph{Clean-field bound.}
\begin{equation}
\label{eq:bound-clean}
\mathbb{E}_S\|\hat f_n - f^\star\|_{L^2(P_X;\mathcal{Y})}^2
\le \varepsilon_{\mathrm{approx}}^2 + \frac{4 C_\mathcal{F}}{\sqrt{n}}.
\end{equation}
\item[(ii)] \emph{HF-GT bound.}
\begin{equation}
\label{eq:bound-hfgt}
\mathbb{E}_S\|\hat f_n - \tilde f\|_{L^2(P_X;\mathcal{Y})}^2
\le 2\varepsilon_{\mathrm{approx}}^2 + \frac{8 C_\mathcal{F}}{\sqrt{n}} + 2\eta_{\mathrm{abs}}^2.
\end{equation}
\item[(iii)] \emph{Crossover.} If
$2\varepsilon_{\mathrm{approx}}^2 + \eta_{\mathrm{abs}}^2 < \sigma_{\mathrm{MC}}^2$,
then for every $n \ge n^\star := (8 C_\mathcal{F}/\Delta)^2$ with
$\Delta := \sigma_{\mathrm{MC}}^2 - 2\varepsilon_{\mathrm{approx}}^2 - \eta_{\mathrm{abs}}^2$,
\begin{equation}
\label{eq:zsd-regime}
\mathbb{E}_S\|\hat f_n - \tilde f\|_{L^2(P_X;\mathcal{Y})}^2
\;\le\; \mathbb{E}\|Y - \tilde f\|_\mathcal{Y}^2.
\end{equation}
\end{enumerate}
\end{theorem}

In words: the learned operator converges to the clean field at the
standard $\mathcal{O}(n^{-1/2})$ rate~\eqref{eq:bound-clean}, and beyond
the finite sample size $n^\star$ its predictions are at least as close
to the high-fidelity reference as the noisy training labels themselves
are~\eqref{eq:zsd-regime}. The crossover condition
$2\varepsilon_{\mathrm{approx}}^2 + \eta_{\mathrm{abs}}^2 < \sigma_{\mathrm{MC}}^2$
formalizes the asymmetric-label-quality regime: the model class must be
expressive enough relative to the high-fidelity reference quality
($2\varepsilon_{\mathrm{approx}}^2 + \eta_{\mathrm{abs}}^2$ small), and the
training noise must be large enough relative to that reference
($\sigma_{\mathrm{MC}}^2$ large). Both sides of this inequality are
controllable: the left through architecture and reference budget, the
right through the IF-GT ray budget.

\begin{proof}[Proof of Theorem~\ref{thm:zsd}]
Fix any $\varepsilon' > 0$ and choose $f_{\varepsilon'} \in \mathcal{F}$
with
$\mathcal{L}(f_{\varepsilon'}) \le \inf_{f \in \mathcal{F}}\mathcal{L}(f) + \varepsilon'$.
By Lemma~\ref{lem:population} applied to $f_{\varepsilon'}$ and to the
infimum,
\begin{equation}
\inf_{f \in \mathcal{F}}\mathcal{L}(f)
= \inf_{f \in \mathcal{F}}\|f - f^\star\|_{L^2(P_X;\mathcal{Y})}^2
+ \sigma_{\mathrm{MC}}^2
= \varepsilon_{\mathrm{approx}}^2 + \sigma_{\mathrm{MC}}^2.
\end{equation}

\textbf{Step 1: Oracle inequality.} Under~(A3)--(A4), the standard
Rademacher symmetrization
bound~\cite{bartlett2002rademacher} gives
\begin{equation}
\mathbb{E}_S\Big[\sup_{f \in \mathcal{F}}\big|\mathcal{L}(f) - \widehat{\mathcal{L}}_n(f)\big|\Big]
\le 2\,\mathfrak{R}_n(\mathcal{L}_\mathcal{F})
\le \frac{2 C_\mathcal{F}}{\sqrt{n}}.
\end{equation}
Since $\hat f_n$ minimizes $\widehat{\mathcal{L}}_n$ over $\mathcal{F}$,
\begin{align}
\mathbb{E}_S[\mathcal{L}(\hat f_n)]
&\le \mathbb{E}_S[\widehat{\mathcal{L}}_n(\hat f_n)] + \frac{2 C_\mathcal{F}}{\sqrt{n}}
\le \mathbb{E}_S[\widehat{\mathcal{L}}_n(f_{\varepsilon'})] + \frac{2 C_\mathcal{F}}{\sqrt{n}} \nonumber\\
&\le \mathcal{L}(f_{\varepsilon'}) + \frac{4 C_\mathcal{F}}{\sqrt{n}}
\le \inf_{f \in \mathcal{F}}\mathcal{L}(f) + \varepsilon' + \frac{4 C_\mathcal{F}}{\sqrt{n}}.
\end{align}

\textbf{Step 2: Clean-field bound.} Apply Lemma~\ref{lem:population} to
$\hat f_n$ and subtract $\sigma_{\mathrm{MC}}^2$:
\begin{equation}
\mathbb{E}_S\|\hat f_n - f^\star\|_{L^2(P_X;\mathcal{Y})}^2
\le \varepsilon_{\mathrm{approx}}^2 + \varepsilon' + \frac{4 C_\mathcal{F}}{\sqrt{n}}.
\end{equation}
Since $\varepsilon' > 0$ was arbitrary, taking $\varepsilon' \to 0$
yields~\eqref{eq:bound-clean}.

\textbf{Step 3: HF-GT bound.}
Decompose
$\hat f_n(X) - \tilde f(X) = (\hat f_n(X) - f^\star(X)) - \delta_{\mathrm{abs}}(X)$
and apply $\|a-b\|^2 \le 2\|a\|^2 + 2\|b\|^2$ pointwise. Taking
expectations and using~\eqref{eq:bound-clean} gives~\eqref{eq:bound-hfgt}.

\textbf{Step 4: Label identity.} Write
$Y - \tilde f(X) = \varepsilon - \delta_{\mathrm{abs}}(X)$ and expand.
The cross-term vanishes by the tower property:
$\mathbb{E}\langle\varepsilon, \delta_{\mathrm{abs}}(X)\rangle_\mathcal{Y}
= \mathbb{E}\langle\mathbb{E}[\varepsilon \mid X], \delta_{\mathrm{abs}}(X)\rangle_\mathcal{Y} = 0$.
Therefore
\begin{equation}
\label{eq:label-id}
\mathbb{E}\|Y - \tilde f(X)\|_\mathcal{Y}^2 = \sigma_{\mathrm{MC}}^2 + \eta_{\mathrm{abs}}^2.
\end{equation}

\textbf{Step 5: Crossover.}
If $2\varepsilon_{\mathrm{approx}}^2 + \eta_{\mathrm{abs}}^2 < \sigma_{\mathrm{MC}}^2$,
set $\Delta := \sigma_{\mathrm{MC}}^2 - 2\varepsilon_{\mathrm{approx}}^2 - \eta_{\mathrm{abs}}^2 > 0$
and $n^\star := (8 C_\mathcal{F}/\Delta)^2$. For $n \ge n^\star$,
$8 C_\mathcal{F}/\sqrt{n} \le \Delta$, so~\eqref{eq:bound-hfgt} gives
$\mathbb{E}_S\|\hat f_n - \tilde f\|_{L^2(P_X;\mathcal{Y})}^2
\le \sigma_{\mathrm{MC}}^2 + \eta_{\mathrm{abs}}^2
= \mathbb{E}\|Y - \tilde f\|_\mathcal{Y}^2$,
which is the regime~\eqref{eq:zsd-regime}.
\end{proof}

\begin{remark}[On the $\eta_{\mathrm{abs}}^2 \ll \sigma_{\mathrm{MC}}^2$ regime]
\label{rem:hfgt-clean}
Combining Remark~\ref{rem:mcrt-noise} with~\eqref{eq:label-id} gives
$\mathbb{E}\|Y - \tilde f\|_\mathcal{Y}^2 \approx \sigma_{\mathrm{MC}}^2$
to leading order, justifying the main-text statement that HF-GT
acts as an effectively noiseless reference relative to IF-GT.
\end{remark}

\subsection{Realizable feature-regression analogue}
\label{app:ols}

The following proposition makes the noise-averaging mechanism
transparent in a finite-dimensional realizable model. It is not used in
the proof of Theorem~\ref{thm:zsd}; we include it because it produces
the same $1/\sqrt{n}$ scaling with explicit constants and clarifies what
the abstract Rademacher bound is doing.

\begin{proposition}[Exact feature regression]
\label{prop:ols}
Suppose $f^\star(x) = \beta^{\star\top}\phi(x)$ for whitened features
$\phi(x) \in \mathbb{R}^m$ with $\mathbb{E}[\phi(X)\phi(X)^\top] = I_m$,
and that labels obey
$Y_i = \beta^{\star\top}\phi(X_i) + \varepsilon_i$ with
$\mathbb{E}[\varepsilon_i \mid X_i] = 0$ and
$\mathbb{E}[\varepsilon_i^2 \mid X_i] \le \sigma^2$. Let
$\hat\Sigma_n := \tfrac1n\sum_i \phi(X_i)\phi(X_i)^\top$ and assume
$\lambda_{\min}(\hat\Sigma_n) \ge \mu > 0$. Then the OLS estimator
$\hat\beta$ satisfies
\begin{equation}
\mathbb{E}[\|\hat\beta - \beta^\star\|_2^2 \mid X_{1:n}] \le \frac{m\sigma^2}{\mu n},
\qquad
\mathbb{E}[\|\hat f - f^\star\|_{L^2(P_X)}^2 \mid X_{1:n}] \le \frac{m\sigma^2}{\mu n},
\end{equation}
where $\hat f(x) := \hat\beta^\top \phi(x)$.
\end{proposition}

\begin{proof}
The normal equations give
$\hat\beta - \beta^\star = \hat\Sigma_n^{-1}(\tfrac1n\sum_i \phi_i \varepsilon_i)$.
Conditioning on $X_{1:n}$, sample independence and conditional zero-mean
make cross-terms vanish, so
\begin{align}
\mathbb{E}[\|\hat\beta - \beta^\star\|_2^2 \mid X_{1:n}]
&= \tfrac{1}{n^2}\sum_i \mathbb{E}[\varepsilon_i^2 \mid X_i]\,\phi_i^\top \hat\Sigma_n^{-2}\phi_i
\le \tfrac{\sigma^2}{n^2}\,\mathrm{tr}\!\big(\hat\Sigma_n^{-2}\textstyle\sum_i \phi_i\phi_i^\top\big) \nonumber \\
&= \tfrac{\sigma^2}{n}\,\mathrm{tr}(\hat\Sigma_n^{-1})
\le \tfrac{m\sigma^2}{\mu n}.
\end{align}
Whitening gives
$\|\hat f - f^\star\|_{L^2(P_X)}^2 = \|\hat\beta - \beta^\star\|_2^2$.
\end{proof}

This shows that even in the simplest setting, with $m$ effective
parameters and per-sample variance $\sigma^2$, the estimation error
decays as $m/(\mu n)$ --- independent of $\sigma^2$ in the same sense as
\eqref{eq:bound-clean}: the noise contributes only through a $\sigma^2$
prefactor that is dominated by the $1/n$ scaling once
$n > m\sigma^2/(\mu \cdot \text{target})$.

\subsection{Discrete Sobolev target preservation}
\label{app:sobolev}

The Sobolev regularizer in Section~\ref{sec:training-obj} adds a
gradient-domain term to the training loss. We show that this does not
shift the population minimizer away from $f^\star$, and we extend the
result to the geometry-weighted form actually used in
Section~\ref{sec:training-obj}.

Let fields be represented on a fixed grid and let $D$ be a fixed linear
finite-difference gradient operator. For $\alpha > 0$, the discrete
Sobolev seminorm is
$\|v\|_{H_\alpha^1}^2 := \|v\|_2^2 + \alpha\|Dv\|_2^2$.

\begin{proposition}[Discrete Sobolev population minimizer]
\label{prop:sobolev}
Under~(A1), for every measurable $f : \mathcal{X} \to \mathcal{Y}$,
\begin{equation}
\mathbb{E}\|f(X) - Y\|_{H_\alpha^1}^2
= \mathbb{E}\|f(X) - f^\star(X)\|_{H_\alpha^1}^2 + \mathbb{E}\|\varepsilon\|_{H_\alpha^1}^2,
\end{equation}
so the population minimizer of the discrete Sobolev loss is $f^\star$.
\end{proposition}

\begin{proof}
Define the linear map $A_0 v := (v, \sqrt{\alpha}\,Dv)$, so that
$\|v\|_{H_\alpha^1}^2 = \|A_0 v\|_2^2$. Then
$A_0(f(X) - Y) = A_0(f(X) - f^\star(X)) - A_0 \varepsilon$. Expanding
the square and taking expectations,
\begin{equation}
\mathbb{E}\|f(X) - Y\|_{H_\alpha^1}^2
= \mathbb{E}\|f(X) - f^\star(X)\|_{H_\alpha^1}^2
- 2\,\mathbb{E}\langle A_0(f(X) - f^\star(X)),\,A_0 \varepsilon\rangle
+ \mathbb{E}\|\varepsilon\|_{H_\alpha^1}^2.
\end{equation}
Because $A_0$ is fixed and linear, it commutes with conditional
expectation: $\mathbb{E}[A_0 \varepsilon \mid X] = A_0\,\mathbb{E}[\varepsilon \mid X] = 0$
by~(A1). The cross-term therefore vanishes by the tower property.
\end{proof}

\begin{corollary}[Weighted Sobolev with $X$-measurable weights]
\label{cor:sobolev-weighted}
Let $W(X) : \mathcal{Y} \to \mathcal{Y}$ be a $\sigma(X)$-measurable
bounded linear operator on $\mathcal{Y}$ (e.g.\ pointwise multiplication
by a bounded geometry-derived field $w_{\mathrm{dist}}(X)$). Then under~(A1),
the weighted Sobolev loss
\begin{equation}
\mathcal{L}_W(f) := \mathbb{E}\|W(X)(f(X) - Y)\|_{H_\alpha^1}^2
\end{equation}
satisfies $\mathcal{L}_W(f) = \mathbb{E}\|W(X)(f(X) - f^\star(X))\|_{H_\alpha^1}^2
+ \mathbb{E}\|W(X)\varepsilon\|_{H_\alpha^1}^2$,
so its population minimizer (in the kernel-quotient sense determined by
$W(X)$) is consistent with $f^\star$. In particular, the
distance-modulated Sobolev loss in Section~\ref{sec:training-obj} does
not bias the population target.
\end{corollary}

\begin{proof}
$W(X) A_0$ is $\sigma(X)$-measurable and linear, so
$\mathbb{E}[W(X) A_0 \varepsilon \mid X] = W(X) A_0\,\mathbb{E}[\varepsilon \mid X] = 0$
by~(A1). The same expansion as in Proposition~\ref{prop:sobolev}
eliminates the cross-term.
\end{proof}

The same argument applied with $\alpha = 0$ and $W(X)$ a
geometry-confidence weight covers the confidence-modulated $L^2$ stage
losses $\mathcal{L}_k$ in Section~\ref{sec:training-obj}.
\begin{revisedblock}
\subsection{\revised{Stage-wise losses inherit the denoising property}}
\label{app:stage-loss}

Theorem~\ref{thm:zsd} is stated for a single target. RayNO instead supervises three cumulative outputs
against three staged labels, with a shared encoder and a weighted Huber objective, so it is worth confirming
that the argument carries over. It does, and the reason is simple: the prediction depends on the scene, not
on the label draw.

Let $\bar{F}_{k}(X) := \mathbb{E}[Y_{k}\mid X]$ be the cumulative stage-$k$ target in the training
representation, including the fixed pre-training conversion, and write
$Y_{k} = \bar{F}_{k}(X) + \varepsilon_{k}$ with $\mathbb{E}[\varepsilon_{k}\mid X]=0$ by
Proposition~\ref{prop:mcrt}. For any bounded weight $w_{k}\ge 0$ derived from the input,
\begin{equation}
\label{eq:stage-weighted-split}
\mathbb{E}\sum_{p} w_{k}\big(f_{\theta,k}-Y_{k}\big)^{2}
=
\mathbb{E}\sum_{p} w_{k}\big(f_{\theta,k}-\bar{F}_{k}\big)^{2} + C_{k},
\qquad
C_{k} := \mathbb{E}\sum_{p} w_{k}\,\varepsilon_{k}^{2},
\end{equation}
where the sum runs over receiver pixels $p$. The cross term vanishes because, conditioned on $X$, both
$f_{\theta,k}$ and $w_{k}$ are deterministic while $\varepsilon_{k}$ is centered. Since $C_{k}$ does not
involve $\theta$, finite-ray error raises the expected loss by a constant and leaves the minimizer where it
was.

Three features of the architecture might be expected to disturb this, and none does. The shared encoder
couples the stages through $\theta$, but Eq.~\eqref{eq:stage-weighted-split} holds for each $k$ separately at
any fixed $\theta$, so summing over stages preserves it. The outputs are cumulative rather than independent,
which again only matters through $\theta$. The errors $\varepsilon_{1},\varepsilon_{2},\varepsilon_{3}$ are
correlated, since the staged labels come from simulations of the same scene, but the identity uses only the
conditional mean of each $\varepsilon_{k}$ and never a product of two of them.

Two remaining gaps are closed elsewhere. The objective uses a Huber rather than a squared loss; measured
against the higher-ray target, more than $99.98\%$ of geometry-weighted Huber residuals fall in the quadratic
region, so the objective behaves as a weighted squared loss over essentially all of the training signal.
Proposition~\ref{prop:sobolev} and Corollary~\ref{cor:sobolev-weighted} extend the same argument to the
squared Sobolev terms of Eq.~\eqref{eq:sob-loss}.
\end{revisedblock}

\begin{revisedblock}
\subsection{\revised{Checking the noise model against the simulator}}
\label{app:seed-validation}

Assumption~(A1) and Proposition~\ref{prop:mcrt} are statements about Sionna RT, not about our model, so they
can be checked directly. This section reports that check. It is relevant because the residual between the
label and the low-fidelity input visibly follows walls and shadow boundaries, which may appear to contradict
the zero-mean assumption.

The two observations are consistent once the gap between the label and the low-fidelity input is split into
two parts,
\begin{equation}
\label{eq:residual-split}
Y_{\mathrm{IF\text{-}GT}} - u = \big(f^{\star}(X) - u\big) + \varepsilon .
\end{equation}
The first part is the propagation structure that the low-fidelity simulation does not compute, so it follows
the scene geometry by design. Only $\varepsilon$ is simulation noise, and Assumption~(A1) asks only that it has
zero conditional mean, not that it is independent of $X$. Writing the finite-ray
field at a fixed scene $X$ as
\begin{equation}
\label{eq:finite-ray-field}
P_{s,m}(X) = \frac{1}{m}\sum_{r=1}^{m} Z_{s,X}(\xi_r),
\end{equation}
Proposition~\ref{prop:mcrt} gives conditional mean equal to the converged stage field and variance decaying as
$1/m$. Independence here refers to the sampled rays \emph{after} the scene is fixed. The error may still
depend on the scene, and on the low-fidelity map, through its variance and its spatial correlation. The
structured residual a reader sees is $Y_{\mathrm{IF\text{-}GT}} - u = [f^\star(X) - u] + \varepsilon$; the
bracketed term is the repeatable reflected, diffracted, and scattered structure missing from the LF map, and
only $\varepsilon$ is the theorem's noise variable.

\paragraph{Protocol.} We drew $512$ independently randomized $10^{6}$-ray realizations per stage on fixed
scenes, varying only the Sionna seed. The comparison reference is a pooled higher-ray field assembled from
separate runs that share no realization with the $512$, so the test is not self-referential. We then measured
how the error of the running average decays as realizations accumulate.

\paragraph{Result.} The averaging-error slopes lie between $-0.493$ and $-0.502$ on a log--log scale, against
the $-0.5$ that an unbiased Monte Carlo estimator predicts. The averaged linear-power fields settle within
$0.21$\,dB of the reference, the largest mean offset across predefined spatial groups is $0.049$\,dB, and the
averages show no plateau over the range tested. A biased estimator would show the opposite signature: the
decay would flatten at the level of the bias instead of continuing. These measurements support conditional
centering while leaving scene-dependent variance and spatial correlation intact, which is all (A1) requires.
The independent-seed averaging experiment of Appendix~\ref{app:if-averaging} is a second view of the same
property, measured in the metrics the paper reports.

\paragraph{Loss geometry.} One detail matters for the weighted Huber objective of
Section~\ref{sec:training-obj}, since the constant-offset argument of Eq.~\eqref{eq:stage-weighted-split} is
stated for a squared loss. More than $99.98\%$ of geometry-weighted Huber residuals fall in the quadratic
region of the Huber function when measured against the higher-ray target, so the objective behaves as a
weighted squared loss over essentially all of the training signal, and Proposition~\ref{prop:sobolev} with
Corollary~\ref{cor:sobolev-weighted} covers the Sobolev terms.
\end{revisedblock}

\newpage
\section*{NeurIPS Paper Checklist}

\begin{enumerate}

\item {\bf Claims}
    \item[] Question: Do the main claims made in the abstract and introduction accurately reflect the paper's contributions and scope?
    \item[] Answer: \answerYes{}. %
    \item[] Justification:  The abstract and introduction state the paper's main contributions: the RayNO architecture, evaluation against baselines and noisy training labels, and the theoretical explanation for denoising without clean references. These claims are supported by the method, empirical results, ablations, and theoretical appendix, and the scope is limited to the stated indoor ray-tracing setting.

\item {\bf Limitations}
    \item[] Question: Does the paper discuss the limitations of the work performed by the authors?
    \item[] Answer: \answerYes{}. %
    \item[] Justification: Appendix A discusses intended use, foreseeable risks, and explicit scope limitations, including single-frequency evaluation, indoor-only scenes, single-transmitter settings, scalar RSS prediction, and reliance on ray-traced supervision. The paper also discusses sensitivity to out-of-distribution scene shifts and biased label noise.

\item {\bf Theory assumptions and proofs}
    \item[] Question: For each theoretical result, does the paper provide the full set of assumptions and a complete (and correct) proof?
    \item[] Answer: \answerYes{}. %
    \item[] Justification: The theoretical appendix states the notation and standing assumptions used in the analysis and provides complete proofs for the denoising-without-clean-references result, the Monte Carlo conditional-unbiasedness argument, the population decomposition lemma, and Sobolev target-preservation extensions.

    \item {\bf Experimental result reproducibility}
    \item[] Question: Does the paper fully disclose all the information needed to reproduce the main experimental results of the paper to the extent that it affects the main claims and/or conclusions of the paper (regardless of whether the code and data are provided or not)?
    \item[] Answer: \answerYes{}. %
    \item[] Justification: The paper specifies the simulation setting, train/test protocol, model inputs, training targets, baselines, metrics, ablations, curriculum, optimization details, and compute footprint. We also provide the complete implementation and dataset links in the supplemental material to reproduce the reported results.

\item {\bf Open access to data and code}
    \item[] Question: Does the paper provide open access to the data and code, with sufficient instructions to faithfully reproduce the main experimental results, as described in supplemental material?
    \item[] Answer: \answerYes{}. %
    \item[] Justification: We provide an anonymized supplemental code archive containing the full RayNO implementation, training and evaluation scripts, baseline scripts, and reproduction instructions. Dataset links are also provided for the generated ray-tracing data used in the main experiments.

\item {\bf Experimental setting/details}
    \item[] Question: Does the paper specify all the training and test details (e.g., data splits, hyperparameters, how they were chosen, type of optimizer) necessary to understand the results?
    \item[] Answer: \answerYes{}. %
    \item[] Justification: The main text describes the experimental protocol and evaluation setting, while the appendix provides the full training details, including data splits, input construction, optimizer, learning-rate schedule, loss weights, batching, precision, random seed, baseline tuning protocol, and evaluation metrics.

\item {\bf Experiment statistical significance}
    \item[] Question: Does the paper report error bars suitably and correctly defined or other appropriate information about the statistical significance of the experiments?
    \item[] Answer: \answerYes{}. %
    \item[] Justification: The extended results appendix reports mean performance with 95\% paired bootstrap confidence intervals for the main metrics supporting the paper's claims. The confidence intervals are computed by resampling held-out test scenes/samples with a fixed number of bootstrap draws.

\item {\bf Experiments compute resources}
    \item[] Question: For each experiment, does the paper provide sufficient information on the computer resources (type of compute workers, memory, time of execution) needed to reproduce the experiments?
    \item[] Answer: \answerYes{}. %
    \item[] Justification: The appendix reports the compute footprint for RayNO training, including the GPU class, training duration, epoch count, batch size, and the fact that no multi-GPU or multi-node infrastructure is required. It also documents the baseline tuning and training protocol so that the compute needed for reproduced comparisons is clear.
    
\item {\bf Code of ethics}
    \item[] Question: Does the research conducted in the paper conform, in every respect, with the NeurIPS Code of Ethics \url{https://neurips.cc/public/EthicsGuidelines}?
    \item[] Answer: \answerYes{}. %
    \item[] Justification: The work uses synthetic ray-traced wireless scenes and does not involve human subjects, private user data, scraped personal data, or deployed decision-making systems. The submission and supplemental artifacts are anonymized and prepared in accordance with the NeurIPS Code of Ethics.

\item {\bf Broader impacts}
    \item[] Question: Does the paper discuss both potential positive societal impacts and negative societal impacts of the work performed?
    \item[] Answer: \answerYes{}. %
    \item[] Justification: The broader impact and limitations section discusses positive impacts such as faster indoor network planning, improved coverage, and lower deployment cost. It also discusses risks, especially over-trust in surrogate predictions outside the evaluated regime, and states the intended operating scope and limitations.
    
\item {\bf Safeguards}
    \item[] Question: Does the paper describe safeguards that have been put in place for responsible release of data or models that have a high risk for misuse (e.g., pre-trained language models, image generators, or scraped datasets)?
    \item[] Answer: \answerNA{}. %
    \item[] Justification: The released artifacts consist of synthetic ray-tracing datasets, training/evaluation code, and neural surrogate models for indoor wireless field prediction. They do not include scraped personal data, human-generated content, pretrained generative models, or other assets with a high risk of direct misuse.

\item {\bf Licenses for existing assets}
    \item[] Question: Are the creators or original owners of assets (e.g., code, data, models), used in the paper, properly credited and are the license and terms of use explicitly mentioned and properly respected?
    \item[] Answer: \answerYes{}. %
    \item[] Justification: The paper cites the existing simulation tools, software libraries, and baseline methods used in the experiments, including the ray-tracing simulator and neural-operator/baseline implementations. The supplemental material documents the relevant dependencies, versions, and license information, and all released code and data are distributed consistently with those terms.

\item {\bf New assets}
    \item[] Question: Are new assets introduced in the paper well documented and is the documentation provided alongside the assets?
    \item[] Answer: \answerYes{}. %
    \item[] Justification: We release new synthetic ray-tracing datasets, code, and reproduction scripts as part of the supplemental material. The documentation describes the dataset structure, generation protocol, train/test splits, expected file formats, model checkpoints or scripts, and instructions for reproducing the main experiments.

\item {\bf Crowdsourcing and research with human subjects}
    \item[] Question: For crowdsourcing experiments and research with human subjects, does the paper include the full text of instructions given to participants and screenshots, if applicable, as well as details about compensation (if any)? 
    \item[] Answer: \answerNA{}. %
    \item[] Justification: The paper does not involve crowdsourcing, user studies, human-subject experiments, or human-generated annotations. All data are generated synthetically through controlled ray-tracing simulations.

\item {\bf Institutional review board (IRB) approvals or equivalent for research with human subjects}
    \item[] Question: Does the paper describe potential risks incurred by study participants, whether such risks were disclosed to the subjects, and whether Institutional Review Board (IRB) approvals (or an equivalent approval/review based on the requirements of your country or institution) were obtained?
    \item[] Answer: \answerNA{}. %
    \item[] Justification:  The research does not involve human subjects, crowdsourcing, personal data collection, behavioral data, or user-facing studies. Therefore, IRB or equivalent human-subjects review is not applicable.

\item {\bf Declaration of LLM usage}
    \item[] Question: Does the paper describe the usage of LLMs if it is an important, original, or non-standard component of the core methods in this research? Note that if the LLM is used only for writing, editing, or formatting purposes and does \emph{not} impact the core methodology, scientific rigor, or originality of the research, declaration is not required.
    \item[] Answer: \answerNA{}. %
    \item[] Justification: LLMs are not used as an important, original, or non-standard component of the proposed method, experiments, data generation, or scientific claims. Any use of language models was limited to ordinary writing, editing, or formatting assistance and did not affect the core methodology or results.

\end{enumerate}

\end{document}